\documentclass[twoside,11pt]{article}

\usepackage{packages}
\newcommand{\Ito}{It\^o}
\newcommand{\dnabla}{\widetilde{\nabla}}
\newcommand{\quantile}{\operatorname{quantile}}
\newcommand{\QuantRegret}{\operatorname{QuantRegret}}
\newtheorem{fact}[theorem]{Fact}

\author{%
 \name{Nicholas J.\ A.\ Harvey} \email{nickhar@cs.ubc.ca}\\
 \addr University of British Columbia\\
 Vancouver, BC, Canada\\
 \name{Christopher Liaw} \email{cvliaw@google.com}\\
 \addr Google\\
 Mountain View, CA, USA\\
 \name{Victor Sanches Portella} \email{victorsp@cs.ubc.ca}\\
 \addr University of British Columbia\\
 Vancouver, BC, Canada%\\
}

\newenvironment{proofof}[1][\unskip]%
{%
\par\noindent{\bfseries\upshape \proofname~of {#1}\ }%
}%
{\jmlrQED}

\newcommand{\qedhere}{\tag*{\jmlrQED}}
\newenvironment{fakeproof}%
{%
\par\noindent{\bfseries\upshape \proofname\ }%
}%
{}

\newenvironment{proofsketch}[1][\unskip]%
{%
\par\noindent{\bfseries\upshape \proofname~sketch\ }%
}%
{\jmlrQED}

\numberwithin{theorem}{section}

\begin{document}

\ShortHeadings{Continuous Prediction with Experts' Advice}{Harvey, Liaw, Sanches Portella}
\title{Continuous Prediction with Experts' Advice}
\editor{}

\maketitle

\begin{abstract}%
    Prediction with experts' advice is one of the most fundamental problems in online learning and captures many of its technical challenges. 
    A recent line of work has looked at online learning through the lens of
    differential equations and continuous-time analysis.
    % Thinking in continuous time has brought us tools that we otherwise would not have considered.
    This viewpoint has yielded optimal results for several problems in online learning.
    
    %continuous mathematics

    In this paper, we employ continuous-time stochastic calculus in order to study the discrete-time experts' problem.
    We use these tools to design a continuous-time, parameter-free algorithm with improved guarantees on the \emph{quantile regret}.
    We then develop an analogous discrete-time algorithm with a very similar analysis and identical quantile regret bounds. Finally, we design an \emph{anytime} continuous-time algorithm with regret matching the optimal \emph{fixed-time} rate when the gains are independent Brownian motions; in many settings, this is the most difficult case.
    This gives some evidence that, even with adversarial gains, the optimal anytime and fixed-time regrets may coincide.
    
    %We hope this to be a showcase of the potential of this continuous-time setting in the study of the experts' problem.
\end{abstract}

%\begin{abstract}%
   % Prediction with expert advice is one of the most fundamental problems in online learning and captures many of its technical challenges. Yet, there are many cases where we still do not know minimax optimal regret. For instance, we do not know whether algorithms that do not know the time-horizon---known as anytime algorithms---can perform as well as those that require a fixed-time. Furthermore, there is still a gap between the best known upper- and lower-bounds on \emph{quantile regret}, the regret against a top quantile of experts rather than the best expert.

   % In this paper, we use a continuous setting for the experts' problem as a proxy to study the discrete problem. In continuous-time we design a parameter-free algorithm that guarantees improved quantile regret bounds. We further show how to discretize this algorithm while preserving the quantile regret bounds. Finally, we design an anytime continuous-time algorithm that achieves regret matching the optimal fixed-time rate when the gains are independent Brownian Motions, suggesting that the optimal anytime regret may match the optimal fixed-time rate. We hope this to be a showcase of the potential of this continuous-time setting in the study of the experts' problem.
%\end{abstract}


\section{Introduction}

% \victor{[Introduce the experts problem]}
One of the cornerstone online learning (OL) tasks is \emph{prediction with experts' advice} or \emph{experts' problem}. In this problem, at each round $t = 1, 2, \ldots$ a player picks a probability distribution \(p_t\) over \(n\) experts.
% at each of a sequence of many rounds. 
Next, an adversary picks gains\footnote{In this paper we use gains in \([-1,1]\) instead of costs in \([0,1]\) due to parallels to random walks and Brownian motion.} \(g_t \in [-1,1]^n\) for each of the experts. At the end of round $t$ the player receives the expected gains \(\iprodt{p_t}{g_t}\) of the experts according to \(p_t\). The performance of the player is usually measured by the \emph{regret}: the difference between the best expert's gains in hindsight and the players' gains. Albeit classical, the experts' problem already captures many of the key theoretical challenges in OL. Determining the minimax optimal regret has been a foundational research vein in OL. We focus on the analysis of optimal rates in two settings: \emph{anytime regret} and \emph{quantile regret}.

An intriguing question is to determine the optimal regret achievable by an \emph{anytime} algorithm, that is, an algorithm that does not have access to the total number \(T\) of rounds.  When the algorithm knows \(T\) beforehand, which we refer to as the \emph{fixed-time setting}, the classical Multiplicative Weights Update (MWU) method \citep{Vovk90,LittlestoneW94} suffers no more than \(\sqrt{2 T \ln n}\) regret, which is optimal in the worst-case~\citep{CesaBianchiFHHSW97a}. For the time being, the best anytime regret guarantee known is \(2\sqrt{T \ln n}\) using MWU with a time-varying step-size~\citep[Theorem~2.4]{Bubeck11a}. It is unknown if for general \(n\) there is an algorithm that guarantees regret smaller than \(2\sqrt{T \ln n}\) or whether one can prove a lower bound strictly better than $\sqrt{2 T \ln n}$.

% \victor{[Mention old and recent research on the problem, but with possibly complicated proofs]}
The classical notion of regret may not always be ideal. For example, one might not mind if the player performs badly when compared to the single best expert if it performs well when compared to some $\eps$-quantile of the top experts, denoted by \(\eps\)-\emph{quantile regret}. The first algorithms aimed provably good quantile regret were proposed by~\citet{ChaudhuriFH09}. Currently, the best-known \(\eps\)-quantile regret guarantees are\footnote{Ignoring low order terms relative to \(\sqrt{T \ln(1/\eps)}\) and multiplying by 2 due to gains being in \([-1,1]\) instead of \([0,1]\).} \( 2\sqrt{3 T \ln(1/\eps)}\) in the fixed-time setting \citep{OrabonaP16} and $4\sqrt{t\ln(1/\eps)}$ in the anytime setting \citep{ChernovV10}.
Recently, \citet{NegreaBCOR21} showed a lower-bound of \(\sqrt{2T \ln(1/\eps)}\) on the \(\eps\)-regret.
Thus, the gap between the upper and lower bounds is $\sqrt{6}$ in the fixed-time setting and $2\sqrt{2}$ in the anytime setting.
% \chris{BTW, I think the Negrea lower bound can be simplified. Suppose there are $2/\eps$ experts and pair them up in pairs.
% Within each pair, choose one expert to have gain $1$ and the other to have gain $-1$. So each pair is a random walk with increments of size $2$.
% Expected value of regret is $0$.
% At time $T$, expectation of largest gap is roughly $2\sqrt{2T\ln(1/\eps)}$ so the expected regret is at least half that, i.e.~$\sqrt{2T\ln(1/\eps)}$.
% If there are $2n$ experts, split them up into $n$ pairs then split them up into $1/\eps$ groups of $\eps n$ pairs each and have each group move the same.
% }

% even if it is possible at all for an anytime algorithm to guarantee \(\sqrt{2 T \ln n}\) regret.

% To special interest for us is a note due to Freund, in which he describes a continuous time framework inspired by portfolio management and shows an an algorithm analogous to NormalHedge with a drastically simpler regret analysis.

% \victor{[In the paper we study a continuous version of the experts' problem and summary of the results]}
% \chris{I do not understand the first sentence. Is it saying that it is special case of something similar to the framework studied by Freund? Or maybe that is it simpler variant of the framework studied by Freund because we are not considering a drift term.}
% \victor{I guess the more accurate description is to say it is a simpler variant of Freund's framework. What I want to avoid is claiming that this is a new framework. I've changed the wording a bit.}
\subsection{Our Contributions}
We present a continuous-time variant of the experts problem and use it to study minimax optimal (quantile) regret rates. Our setting is a simpler variant of the framework proposed by \cite{Freund09a}, but we use it as a guide in algorithm design. Namely, working in continuous time allows us to utilize powerful analytical tools from stochastic calculus, which often allow for simpler analyzes.
%and analysis.
% The continuous-time framework comes equipped with powerful analytical tools from stochastic calculus, which often allow for simpler analyzes and can be used as a guide for algorithm design.
% In addition, the continuous setting itself has many interesting theoretical questions.
% In addition, the continuous setting also has theoretical questions which are interesting, in and of itself.
% Moreover, the continuous setting not only has theoretical questions that are interesting in itself, but its connections to the experts' problem lead to new insights and algorithms to the classical experts' problem.
In this paper, we use continuous-time techniques to obtain improved bounds on the minimax optimal \(\eps\)-quantile regret and obtain intriguing results for the anytime continuous regret. Furthermore, we hope this to be a showcase of the potential of the impact of this continuous-time framework in research on the experts' problem. Our specific results are as follows.
\begin{itemize}
\setlength\itemsep{0.2em}
    \item \textbf{Continuous MWU.} To demonstrate the parallels between the discrete and continuous time problems, we describe a continuous-time version of MWU. In \Theorem{ContinuousMWURegret}, we show that we can easily obtain bounds on the continuous regret that match the best-known regret bounds for MWU: \(\sqrt{2T \ln n}\) in the fixed time setting and \(2 \sqrt{T \ln n}\) in the anytime setting.
    \item \textbf{Continuous quantile regret.} Taking inspiration from \Ito's formula from stochastic calculus, we propose a new algorithm for quantile regret in continuous time. This algorithm has anytime continuous quantile regret bounds whose leading constants are better than any known results for the discrete-case (\Theorem{LooseM0Regret}).
    The bounds hold for all $\eps \in (0, 1)$ and $T > 0$ simultaneously.
    The algorithm can be interpreted as ``parameter-free'', since it does not involve a learning rate.

    \item \textbf{Discretized quantile regret bound.} Next, we discretize the algorithm from the previous section while preserving the anytime quantile regret guarantees (\Theorem{eps_quantile_regret}).
    This algorithm is also parameter-free, and improves upon the best-known\footnote{In independent work, \citet{ZhangCP22} developed an algorithm using coin-betting and a similar potential function to the one we use and which yields a similar quantile regret bound. 
    We further discuss how to obtain the bound from their results in~\Section{DiscreteQuantileRegret}.
    In contrast, our analysis is quite self-contained, and avoids using the coin-betting framework.
    } quantile regret bounds in the literature.
    % Thus, we present an algorithm for the classical experts' problem that improves on the best-known quantile regret bounds known so far.
    Furthermore, our analysis closely matches the continuous-time analysis.
    
    \item \textbf{Improved anytime continuous regret with independent experts.}
    We design an anytime continuous-time algorithm with $\sqrt{2 T \ln n}$ regret (a.s. for all $T$),
    asymptotically in $n$, when the gains are independent Brownian motions (\Theorem{CtsOptRegret}). A simple argument shows that this is optimal (\Proposition{ContinuousLB}): for any algorithm and \emph{fixed} time $T > 0$, the expected regret at time $T$ exceeds $\sqrt{2 T \ln n}(1-o(1))$.
    Thus, against independent experts, the anytime setting is no harder than the fixed-time setting.
    This shows that independent experts, source of many of the known lower-bounds on regret, are not enough to prove tighter lower-bounds in the anytime setting in continuous time. This can also be seen as evidence that $\sqrt{2 T \ln n}$ anytime regret against all adversaries might be possible, matching the optimal fixed-time regret.
\end{itemize}

\subsection{Related Work}

\paragraph{Optimal regret in fixed and anytime settings.}

The most well-known algorithm for the experts setting is the Multiplicative Weights Update (MWU) algorithm \citep{LittlestoneW94, Vovk90}.
In the fixed-time setting (with gains in $[-1, 1]$), MWU achieves a regret bound of $\sqrt{2t \ln n}$ and this bound is tight \citep[Corollary~3.2.2]{CesaBianchiFHHSW97a}.
In the anytime setting, MWU with a dynamic step size is known to achieve a regret bound of $2\sqrt{T \ln n}$ for all times $t \geq 0$ (e.g.~\citealp[\S 14]{Cesa-BianchiL06a}, \citealp[Theorem 4]{Nesterov09a}, and \citealp[\S 2.5]{Bubeck11a}).
It is not known whether the constant $2$ is tight; the best known lower bound for the anytime setting is the $\sqrt{2t\ln n}$ which is inherited from the fixed-time setting.

In the fixed-time setting, the minimax regret is known for $n = 2, 3, 4$ experts~\citep{Cover67a, AbbasiYadkoriB17a, BayraktarEZ20a}.
% For $n = 2$, \cite{Cover67a} showed that the optimal regret is $\sqrt{2T/\pi}$.
% More recently, \cite{AbbasiYadkoriB17a} proved that the optimal regret for $n = 3$ is $\sqrt{32T / 9 \pi}$
% and \cite{BayraktarEZ20a} proved that the optimal regret for $n = 4$ is $\sqrt{\pi T / 2}$.
In the anytime setting, we know an optimal algorithm only for $n = 2$ experts, where \cite{HarveyLPR20a} showed that the optimal regret is $\gamma \sqrt{t}$ where $\gamma \approx 1.3069$.

Another model for regret introduced by~\citet{GravinPS16a} is the geometric stopping time model in which the number of rounds is a geometric random variable.
There is a growing body of work in exploring connections between PDEs and the expert problems in the pursuit of optimal algorithms \citep{AndoniP13,BayraktarEZ20a,BayraktarEZ20b,Drenska17,DrenskaK20,KobzarKW20a}.
Recently, \cite{ZhangCP22} also used PDE techniques to obtain an optimal algorithm for unconstrained online linear optimization.

\paragraph{Quantile regret.}
% \chris{$\eps$ quantile regret introduced by \cite{ChaudhuriFH09}. They obtained a bound of $\sqrt{T\ln(1/\eps)(1 + \ln^3 n)}$.
%     A bound of $\sqrt{T (\ln \ln T + \ln 1/\eps)}$ was proved by a number of people including \cite[Theorem 8]{ChernovV10}, \cite[Theorem 4]{KoolenV15}, \cite[Theorem 1]{LuoS15}.
%     There are also a number of papers to achieve a bound of $\sqrt{T \ln(1/\eps)}$ \cite[Theorem 9]{ChernovV10}, \cite[Corollary 6]{OrabonaP16}, \cite[Corollary 2]{NegreaBCOR21}.
%     \cite[Example 4.5]{FosterRS15} prove $\sqrt{T(\ln T + \ln 1/\eps)}$. (Why does \cite{OrabonaP16} say that \cite{FosterRS15} get $\sqrt{T \ln 1/\eps}$?)
% }
\cite{ChaudhuriFH09} introduced the notion of $\eps$-regret where instead of comparing with the best expert,
one compares with the $\lceil \eps n \rceil$-th best expert (amongst $n$ total experts).
They devised the NormalHedge algorithm which they prove has an $\eps$-quantile regret of $O(\sqrt{t \ln(1/\eps)} + \ln^2 n)$ in \(t\) rounds.
Moreover, the bound holds for all $\eps, t$ simultaneously.
A somewhat different bound of $O(\sqrt{t(\ln \ln t + \ln 1/\eps)})$ was proved by \cite{LuoS15} and \cite{KoolenV15}.
All of these works make use of a potential function to control the regret.
Our work also makes use of a potential function inspired by the work of~\citet{HarveyLPR20a} which may be somewhat reminiscent of the potentials used by \citet{ChaudhuriFH09} and \citet{LuoS15}.

It is possible to improve upon the above bounds.
Indeed, \citet[Theorem 3]{ChernovV10}, \citet[Example 5.1]{FosterRS15}, \citet[Corollary 6]{OrabonaP16}, and \citet[Corollary 2]{NegreaBCOR21} show that it is possible to obtain an
$\eps$-quantile regret of $O(\sqrt{T \ln(1/\eps)})$.
This turns out to be tight up to constant factors \citep[Theorem 1]{NegreaBCOR21}.
We note that \cite{FosterRS15}, \cite{OrabonaP16}, and \cite{NegreaBCOR21} derive regret bounds which depend on the KL divergence between a known prior and the player's probability distribution at a specific point of time;
$\eps$-quantile regret bounds can be recovered as a special case of such bounds.
In this paper, we also recover the $O(\sqrt{T \ln(1/\eps)})$ bound on the $\eps$-quantile regret although, as we shall see, we obtain an improved
constant in front of the $\sqrt{T \ln(1/\eps)}$ term.

\paragraph{Connection to Stochastic Control.} The continuous-time setting we consider in this paper has a similar formulation to problems in optimal control~\citep{Bertsekas05a}. In this field, a classical approach is deriving the optimal control via the Hamilton-Jacobi-Bellman (HJB) equations~(for example see~\citealp[Section~3.2]{Bertsekas05a}). In fact, for the experts' problem in discrete time similar techniques were already used, such as in~\cite{CesaBianchiFHHSW97a}, but even then analyzing the regret of such an optimal policy is challenging or impossible. For the continuous-time setting, the HJB equations would involve a non-smooth value function and then one would need to resort to viscosity solutions. There is an emerging line of work using such techniques to obtain regret bounds, such as the recent solution for the minimax regret with 4 experts~\citep{BayraktarEZ20a}. Yet, these techniques often rely on knowing the time-horizon, and our focus in this paper is to analyze anytime regret bounds. It is still an interesting direction of future research to investigate the use of classical control techniques in the continuous experts' problem.

\subsection{Basic Notation}
We use $[n]$ to denote the set $\{1, \ldots, n\}$.
For a predicate $P$, we write $\boole{P}$ to be $1$ if $P$ is true and $0$ otherwise.
Moreover, if $\boole{P}$ is multiplying an invalid expression (such as one with a division by $0$) and $P$ is false,
we consider the whole expression to be $0$.
Set \([\alpha]_+ \coloneqq \max\{\alpha, 0\}\) for all \(\alpha \in \Reals\).
We use $\ones \in \bR^n$ to denote the all-ones vector and $e_i \in \Reals^n$ for $i \in [n]$ the indicator vector given by $e_i(j) \coloneqq \boole{i = j}$ for all $j \in [n]$.
We denote the $(n-1)$-dimensional probability simplex by $\Delta_n \coloneqq \setst{p \in [0,1]^n}{\iprodt{\ones}{p} = 1}$.
For partial derivatives, we write $\partial_i \coloneqq \partial_{x_i}$ and $\partial_{ij} \coloneqq \partial_{x_i,x_j}$.
Lastly, for $x \in \bR^n$ and $\eps \in (0, 1)$, we write
\begin{equation}
\EquationName{QuantileDef}
\quantile(\eps, x) = x_{\pi(\lceil \eps n \rceil)}~\text{where
$\pi\colon [n] \to [n]$ is any permutation with $x_{\pi(1)} \geq \ldots \geq x_{\pi(n)}$.}
\end{equation}
\comment{
\begin{itemize}
    \item \([n] \coloneqq \{1, \dotsc, n\}\)
    \item \(\boole{P}\) is 1 if \(P\) is true, \(0\) otherwise. Moreover, if \(\boole{P}\) is multiplying a invalid expression (such as one with a division by 0) but \(P\) is false, we consider the whole expression to be zero.
    \item For all \(i \in [n]\), we have \(e_i \in \Reals^n\) given by \(e_i(j) \coloneqq \boole{i = j}\) for all \(j \in [n]\).
    \item \(\ones \in \Reals^n\) is the all-ones vector.
    \item \(\Delta_n \coloneqq \setst{p \in [0,1]^n}{\iprodt{\ones}{p} = 1}\)
    \item \(\partial_i \coloneqq \partial_{x_i}\) and \(\partial_{x_i, x_j} \coloneqq \partial_{ij}\)
    \item \(\quantile(\eps, x) \coloneqq x_{\pi(\ceil{\eps n})}\) where \(\pi \colon [n] \to [n]\) is a permutation such that \(x_{\pi(1)} \geq x_{\pi(2)} \geq \dotsm \geq x_{\pi(n)}\)
    
    \paragraph{Notation for partial derivatives.} 
\end{itemize}
}

\section{The Continuous Prediction Problem}
\SectionName{ContinuousExperts}

As in the discrete experts' problem, we have a number \(n \in \Naturals\) of experts to choose from. In the continuous time setting, we model the cumulative gain of each expert as a mixture of \(n\) independent Brownian motions, as done by~\citet{Freund09a}.

% For the sake of simplicity, we defer the technical details of the setting to \Appendix{ContinuousExperts} and give here the main objects of the setting together with an idea on how it relates to the original experts' problem.
More formally, let \(B_1, \dotsc, B_n\) be \(n\) independent standard Brownian motion processes (or, equivalently, let \(B\) be a \(n\)-dimensional standard Brownian motion). The \textbf{cumulative gain process} \(G_i(t)\) of expert \(i \in [n]\) is given by the following stochastic differential equation (SDE)
% \chris{Is this a stochastic differential equation? If so, can we write ``is given by the following stochastic differential equations (SDE)''?}
% \victor{Again, great suggestion! Thanks, Chris!}
\begin{equation*}
    \diff G_{i}(t) \coloneqq \sum_{j = 1}^n w^{(i)}_{j}(t) \diff B_j(t)
    \eqqcolon \iprod{w^{(i)}(t)}{ \diff B(t)}, \qquad \forall t \geq 0, \forall i \in [n],
\end{equation*}
where \((w^{(i)}(t))_{t \geq 0}\) is any continuous stochastic process\footnote{Any stochastic process we mention in this paper is adapted to the filtration generated by \((B(t))_{t \geq 0}\)} in \(\Reals^n\), not necessarily non-negative, such that \(\norm{w^{(i)}(t)}_2 = 1\) at all times~\(t \geq 0\). For example, if \(w^{(i)}(t) = e_i\) for all \(i \in [n]\) and \(t \geq 0\), then \(G^{(i)}(t)\) is an independent Brownian motion for each \(i \in [n]\). An analogous situation in discrete time would be each expert receiving $\{\pm 1\}$ gain uniformly at random at each step, so each cumulative gain would be a standard random walk.\footnote{ Intuitively, that is one of the reasons results in this setting mirrors the discrete-time case with costs in \([-1,1]\).} In our analysis the ``instantaneous covariance matrix'' \(\Sigma(t)\) between the gain processes will be prominent. Formally, we define \(\Sigma(t) \in \Reals^{n \times n}\) by
\begin{equation*}
    \Sigma_{ij}(t) \coloneqq \iprod{w^{(i)}(t)}{w^{(j)}(t)}, \qquad \forall t \geq 0, \forall i,j \in [n].
\end{equation*}
From its definition, we have that \(\Sigma(t)\) is a positive semi-definite matrix with ones along its diagonal.

Next, we define what a player strategy is in continuous-time and its corresponding regret.
A \textbf{player (strategy)} is a left-continuous\footnote{One might loosen this to only assuming \((p(t))_{t \geq 0}\) is predictable. For a discussion, see \Appendix{predictability}.} process \((p(t))_{t \geq 0}\) on \(\bR^n\) such that \(p(t) \in \Delta_n\) for all \(t \geq 0\), where \(\Delta_n\) is the \((n-1)\)-dimensional simplex. The \textbf{player gain} process \((A(t))_{t \geq 0}\) is given by
\begin{equation*}
 \diff A(t) \coloneqq \sum_{i = 1}^n p_i(t) \diff G_i(t) = \iprod{p(t)}{\diff G(t)}.
\end{equation*}
Moreover, the (continuous) \textbf{regret vector process} is given by
\begin{equation*}
    R_i(t) \coloneqq G_i(t) - A(t), \qquad \forall i \in [n], \forall t \geq 0.
\end{equation*}
That is, \(R_i(t)\) is the regret---in the online learning sense---of the player with respect to expert \(i\). Finally, the \textbf{continuous regret} (of the player strategy $(p(t))_{t \geq 0}$) is
\begin{equation*}
    \ContRegret(t) \coloneqq \max_{i \in [n]} R_{i}(t)
    = \max_{i \in [n]} G_i(t) - A(t)
\end{equation*}
Also, define the \textbf{continuous \(\eps\)-quantile regret} to be \(\QuantRegret(\eps,t) \coloneqq \quantile(\eps, R(t))\). In this paper we will investigate player strategies that guarantee bounds on the (quantile) regret that hold almost surely, that is, with probability 1.

\section{A Continuous Multiplicative Weights Update Method}
\SectionName{ContMWU}

In this section, we describe a continuous-time version of the classical Multiplicative Weights Update (MWU) method. 
This serves as a way to introduce some of our technical tools while avoiding the complexities we later introduce in the choice of potential function. Furthermore, we show bounds on its continuous regret that exactly match the bounds that the discrete algorithm enjoys, giving evidence of the parallels between the discrete and continuous time settings.

Analogous to the discrete version of MWU, we want the probability mass of an expert \(i\) at time \(t\) to be proportional to \(\exp(\eta_t G_i(t))\), where \(\eta_t\) is some positive learning rate that is non-increasing in \(t\).
A familiar approach (see, e.g., \citealp[page 14]{Cesa-BianchiL06a} and \citealp[\S 2.5]{Bubeck11a}) is to use the \emph{LogSumExp} function given by
\begin{equation}
    \label{eq:logsumexp_def}
    \Phi(t,x) \coloneqq \frac{\boole{\eta_t > 0}}{\eta_t} \log\Big( \sum_{i = 1}^n e^{\eta_t x_i} \Big) \quad \text{with}~\eta_t \geq 0,
    \qquad \forall t \geq 0, \forall x \in \Reals^n.
\end{equation}
In our case, the main property that we shall use from \(\Phi\) is that \(\nabla_x \Phi(t, \cdot)\) is the \emph{softmax} function. That is, \(\nabla_x \Phi(t, x) \in \Delta_n\) and $(\nabla_x \Phi(t,x))_i \propto \exp(\eta_t x_i)$,
% lies in the simplex
% and its \(i\)th entry is proportional to \(\exp(\eta_t x_i)\).
% \begin{equation*}
%     \nabla_x \Phi(t, x)_i = \frac{\exp(\eta_t x_i)}{\sum_{j = 1}^n \exp(\eta_t x_j)}, \qquad \forall i \in [n].
% \end{equation*}
which is exactly the probability mass MWU places on expert \(i\) with cumulative gain \(x_i\). Thus, we define the player strategy \((p(t))_{t \geq 0}\) by 
\begin{equation}
    \label{eq:mwu_p}
    p(t) \coloneqq \nabla_x \Phi(t, G(t)), \qquad \forall t \geq 0.
\end{equation}
To analyze the regret of \(p\), we need a way to handle \(A(t) = \int_0^t \iprod{p(s)}{\diff G(s)}\). This is a stochastic integral, so we may use \Ito's formula (Theorem~\ref{thm:itos_formula}), which one can think of as the analogue of the fundamental theorem of calculus for stochastic integrals.

\begin{theorem}[{\Ito's Formula,~\citealp[Theorem~IV.3.3]{RevuzY99a}}]
    \label{thm:itos_formula}
    Let \(F \colon \Reals \times \Reals^n \to \Reals\) be continuously differentiable in its first argument and twice continuously differentiable in its second argument and let \((X_t)_{t \geq 0}\) be a continuous semimartingale in \(\Reals^n\). Then, for any \(T \geq 0\)
    \begin{align*}
        F(T, X(T)) - F(0, X(0)) 
        = \int_{0}^T &\iprod{\nabla_x F(t, X(t))}{\diff X(t)}
        + \int_{0}^T \partial_t F(t, X(t)) \diff t
        \\
        &+ \frac{1}{2} \int_{0}^T \sum_{i,j \in [n]} \partial_{x_i, x_j} F(t, X(t)) 
        \diff [X_i,X_j]_t
    \end{align*}
\end{theorem}

In the third derivative above we use the bracket notation: for two continuous (local) martingales \(M\) and \(N\), the process \([M,N]\), denoted as the \textbf{bracket of \(M\) and \(N\)}, is the unique adapted continuous process such that \(MN - [M,N]\) is a local martingale~\citep[Theorem~IV.1.9]{RevuzY99a}. 

% In \Appendix{Ito} we discuss how to compute these terms somewhat mechanically using ``box calculus'', which is the heuristic that \(\diff [M, N]_t = \diff M(t) \cdot \diff N(t)\). Let us now use \Ito's formula and box calculus to prove the following lemma on the regret of continuous-time MWU.

In general, computing the bracket of two stochastic processes may be non-trivial, which makes the application of \Ito's formula more challenging. Luckily, all the process we deal with are defined as stochastic integrals with respect to other continuous martingales which, ultimately, depend only on integrals with respect to Brownian motions. For example, \(A(t)\) is defined as (a sum of) stochastic integrals of a left-continuous and bounded function \(p(t)\) with respect to the process \(G(t)\). The latter is also a continuous martingale since it is a stochastic integral of a continuous and bounded function \(w^{(i)}(t)\) with respect to the Brownian motion \(B(t)\), which is also a martingale. This is specially useful to compute the bracket of two of these processes. More specifically, using that 
\begin{equation}
    \label{eq:BM_diff}
    [B_i,B_j]_t = 
    \begin{cases}
        0 &\text{if}~i \neq j,\\
        t &\text{if}~i = j,
    \end{cases} 
\end{equation}
we can compute the bracket of martingales by use of formal rules to manipulate the differential symbols (see~\citealp[Remark~14.2.7]{CohenE15a} or~\citealp[Theorem~4.1.2]{Oksendal03a}). More specifically, for two continuous martingales \(M\) and \(N\), we have
\begin{equation*}
    \diff [M,N]_t = \diff M(t) \cdot \diff N(t),
\end{equation*}
and for the right-hand side above we usually can expand according to our definitions. In our case, we can always expand these expressions until they are written only in terms of the differentials of Brownian motions, and such expressions can be simplified using~\eqref{eq:BM_diff}. Let us know combine these rules together with \Ito's formula to analyze to regret of continuous MWU.

% \chris{What does $\Phi(t,x) \approx \max_i x_i$ have to do with \Lemma{CostPotentialRegret}?}
% \victor{We use \Ito to expand \(A(t)\), and the \(- \Phi(T, R(T))\) term ``cancels out'' (actually, is non-positive) the \(\max_i G_i(t)\) term. Before it was an equation, but it should be just a bound, so thanks for catching that! I've also changed the wording above a bit}
\begin{lemma}
    \LemmaName{CostPotentialRegret}
    Let \(\Phi\) be defined as in~\eqref{eq:logsumexp_def} and \(p\) be as in~\eqref{eq:mwu_p}. Then, almost surely,
    \begin{align*}
        \ContRegret(T) \leq  \Phi(0,0) +  \int_0^T \Big(  \partial_t \Phi(t, G(t)) + \frac{1}{2} \sum_{i,j \in [n]} \partial_{ij} \Phi(t, G(t))\Sigma_{ij}(t) \Big) \diff t.
    \end{align*}
\end{lemma}
\begin{proof}
    For all \(t \geq 0\) and \(i,j \in [n]\),
    \begin{align*}
        \diff[G_i,G_j]_t &=
        \diff G_i(t) \cdot \diff G_j(t)
        = \iprod{w^{(i)}(t)}{\diff B(t)} 
        \cdot \iprod{w^{(j)}(t)}{\diff B(t)} \\
        &= \sum_{k = 1}^n w_k^{(i)}(t) \cdot w_k^{(j)}(t) \diff t
        = \Sigma_{ij}(t) \diff t,
    \end{align*}
    where the second to last equation follows from~\eqref{eq:BM_diff} and in the last equation we use the fact that~\(\Sigma_{ij}(t) = \iprod{w^{i}(t)}{w^{(j)}(t)}\). 
    
    Let \(T > 0\). \Ito's formula (\Theorem{itos_formula}) allows us to express \(A(T)\) as
    \begin{align*}
        A(T) &= \int_{0}^T \iprod{\nabla_x\Phi(t, G(t))}{\diff G(t)} 
        \\
        &= \Phi(T, G(T)) - \Phi(0,0) 
        - \int_0^T \Big(  \partial_t \Phi(t, G(t)) + \frac{1}{2} \sum_{i,j \in [n]} \partial_{ij} \Phi(t, G(t)) \Sigma_{ij}(t) \Big) \diff t,
    \end{align*}
    Thus,
    \begin{align*}
        \ContRegret(T) = &\max_{i \in [n]} G_i(T) - \Phi(T, G(T)) + \Phi(0,0)
        \\
        &+  \int_0^T \Big(  \partial_t \Phi(t, G(t)) + \frac{1}{2} \sum_{i,j \in [n]} \partial_{ij} \Phi(t, G(t))\Sigma_{ij}(t) \Big) \diff t.
    \end{align*}  
    % where \(\Lambda \in \bR^{n \times n}\) is a positive semidefinite matrix given by 
    % \begin{equation*}
    %     \Gamma \coloneqq  \ones (\ones - p(t))^\transp \Sigma (\ones - p(t)) \ones^\transp.
    % \end{equation*}
    Finally, recall that the LogSumExp function smoothly approximates the maximum function, that is, 
    \begin{equation*}
        \max_{i \in [n]} x_i \leq \Phi(T, x) \leq \max_{i \in [n]} x_i + \frac{\log n}{\eta_T},
        \qquad \forall x \in \Reals^n.
    \end{equation*}
    This implies that $ \max_{i \in [n]} G_i(T) - \Phi(T, G(T)) \leq 0$.
\end{proof}

At this point, to bound the continuous regret of \((p(t))_{t \geq 0}\) it suffices to bound the partial derivatives of \(\Phi\).
\Lemma{MWUPartialsBounds} bounds these partial derivatives in terms of a tunable learning rate $\eta_t$;
minimizing the regret bound boils down to optimizing $\eta_t$.
We defer the proof of the following lemma to \Appendix{ContMWU} since it essentially follows from simple properties of the LogSumExp function.
% as we do in the next lemma, and then set \(\eta_t\) to minimize the regret bounds. We defer the proof of the next lemma to \Appendix{ContMWU} since it boils down to using simple properties of~\(\Phi\).
\begin{lemma}
    \LemmaName{MWUPartialsBounds}
    Let \(\Phi\) be as in~\eqref{eq:logsumexp_def}. Let \(\eta_t\) be either constant in \(t\) or of the form \(\frac{c}{\sqrt{t}}\), with \(c > 0\). Then,
    \begin{equation*}
        \frac{1}{2} \sum_{i,j \in [n]} \partial_{ij} \Phi(t, x) \Sigma_{ij}
        \leq \frac{\eta_t}{2} 
        \quad \text{and} \quad \partial_t \Phi(t, x) \leq \frac{\log n}{2 t \eta_t}, \qquad \forall t \geq 0, \forall x \in \Reals^n.
    \end{equation*}
\end{lemma}

\Theorem{ContinuousMWURegret} summarizes the continuous regret bounds for MWU with properly chosen learning rates, both for the fixed-time and anytime settings. Crucially, these regret bounds match the best known regret bounds for the discrete-time MWU method (see~\citealp[Theorems 2.1 and 2.4]{Bubeck11a}). 

\begin{theorem}
    \TheoremName{ContinuousMWURegret}
    Let \(\Phi\) be as in~\eqref{eq:logsumexp_def} and \(T\) be a positive number. If \(\eta_t \coloneqq \sqrt{\ln n/2T}\) for all \(t \geq 0\), then \(\ContRegret(T) \leq \sqrt{2 T \ln n}\) almost surely. If \(\eta_t \coloneqq \boole{t > 0} \sqrt{\ln n/t}\) for all \(t \geq 0\), then, almost surely, \(\ContRegret(t) \leq 2 \sqrt{t \ln n}\) for all \(t \geq 0\).
\end{theorem}
\begin{proof}
    Let us first consider the fixed-time case, that is, \(\eta_t \coloneqq \sqrt{2 \ln n/T}\) for all \(t \geq 0\). In this case we have \(\partial_t \Phi(t,x) = 0\) since \(\Phi(\cdot,x)\) is constant for any \(x \in \Reals^n\). Moreover, note that \(\Phi(0,0) = \ln n/\eta_0 = \ln n/\eta_T\). Combining this with Lemmas~\ref{lem:CostPotentialRegret} and~\ref{lem:MWUPartialsBounds}, we have
    \begin{equation*}
        \ContRegret(T) \leq \Phi(0,0)
        +  \int_0^T \frac{1}{2} \sum_{i,j \in [n]} \partial_{ij} \Phi(t, G(t))\Sigma_{ij} \diff t
        \leq \frac{\ln n}{\eta_T}
        + \frac{\eta_T T}{2} = \sqrt{2 T \ln n}.
    \end{equation*}

    Let us now consider the anytime case, that is, when \(\eta_t \coloneqq \boole{t > 0} \sqrt{\ln n/t}\) for all \(t \geq 0\). In this case we have \(\Phi(0,0) = 0\), but \(\partial_t \Phi(t,x)\) is not necessarily 0 anymore. By \Lemma{MWUPartialsBounds}, we have
\begin{equation*}
    \partial_t \Phi(t, G(t)) + \frac{1}{2} \sum_{i,j \in [n]} \partial_{ij} \Phi(t, G(t)) \Sigma_{ij} 
    \leq \frac{\log n}{2 t \eta_t} + \frac{\eta_t}{2} = \sqrt{\frac{\log n}{t}}, \qquad \forall t \geq 0.
\end{equation*}
Thus, for all \(t \geq 0\), we have
$\ContRegret(t) \leq \int_0^t \sqrt{\frac{\log n}{s}} \diff s
\leq 2 \sqrt{t \log n}$.
\end{proof}

It is intriguing that these bounds on the continuous regret differ by a factor of $\sqrt{2}$, exactly as in the discrete experts' problem. A natural question is whether there is an anytime algorithm that enjoys continuous regret bound smaller than \(2 \sqrt{t \log n}\). This is the topic we investigate in \Section{IndependentExperts}.

\section{Quantile Regret Bounds with the Confluent Hypergeometric Potential}
\SectionName{QuantileRegretM0}

In this section, we design a different algorithm for the continuous prediction problem. We choose a potential function inspired by \Ito's formula and obtain quantile regret bounds that are better than the ones known with a relatively simple proof. Furthermore, we show that this strategy has a simple discretization and obtain an algorithm with the same bounds for the discrete experts' problem. In \Section{IndependentExperts} we shall see how a similar algorithm suggests an intriguing result for the anytime setting.

This time around we analyze a player strategy parameterized by a function of \(R(t)\) instead of \(G(t)\). That is, let \(\Phi \colon \Reals_{\geq 0} \times \Reals^n \to \Reals\) be a continuously differentiable function, which we refer to as a \emph{potential function}. We consider the player strategy \((p(t))_{t \geq 0}\) given by\footnote{Throughout this paper all entries of \(\nabla_x \Phi(t, R(t))\) have the same sign, implying that \(p(t) \in \Delta_n\). This \(p(t)\) can be discontinuous when $\nabla_x \Phi(t, R(t))=0$, so we need to ensure it is predictable. This issue is discussed in \Appendix{predictability}. }
\begin{equation}
    \label{eq:potential_player}
    p(t) \coloneqq
    \frac{1}{\iprodt{\ones}{\nabla_{x} \Phi(t, R(t))}} \nabla_{x} \Phi(t, R(t)),
     \qquad \forall t \geq 0,
\end{equation}
setting \(p(t) \coloneqq \frac{1}{n}\ones\) when $\iprodt{\ones}{\nabla_{x} \Phi(t, R(t))} = 0$.
This class of player strategies mimics the potential-based strategies from the discrete experts' problems~\citep[Chapter~2]{Cesa-BianchiL06a}. As in the discrete case, if \(\Phi\) is the LogSumExp potential from~\eqref{eq:logsumexp_def}, we obtain the same player strategy of the last section. In the next lemma we use \Ito's formula to get a useful expression for \(\Phi(T, R(T))\) that, in turn, will guide us in the choice of~\(\Phi\).

% 11
\begin{lemma}
    \LemmaName{RegretPlayerPotFormula}
    Let \(\Phi \colon \Reals_{\geq 0} \times \Reals^n \to \Reals\) be one time continuously differentiable on its first argument and two-times continuously differentiable on its second argument. Let the player strategy \((p(t))_{t \geq 0}\) be as in~\eqref{eq:potential_player}. Then, almost surely for all \(T \geq 0\) we have
    \begin{equation}
        \label{eq:generalPlayerPotFormula}
        \Phi(T, R(T)) - \Phi(0,0)
        = \int_0^T \paren[\big]{\partial_t\Phi(t, R(t)) + \frac{1}{2}\sum_{i,j \in [n]} \partial_{ij} \Phi(t, R(t)) (e_i - p(t))^\transp \Sigma(t) (e_j - p(t))} \diff t.
    \end{equation}
    In particular, if for all \(t \geq 0\) and \(x \in \Reals^n\) we have \(\partial_{ij} \Phi(t,x) = 0\) for all distinct \(i,j \in [n]\) and \(\partial_{ii}\Phi(t,x) \leq 0\) for each \(i \in [m]\), then almost surely for all \(T \geq 0\) we have
    \begin{equation}
        \label{eq:PotLowerBound}
        \Phi(T, R(T)) - \Phi(0,0)
        \geq \int_0^T \paren[\Big]{\partial_t\Phi(t, R(t)) + 2\sum_{i = 1}^n \partial_{ii} \Phi(t, R(t))} \diff t.
    \end{equation}
\end{lemma}
\begin{proof}
Let \(T \geq 0\). \Ito's formula gives us a useful formula to compute the evolution of the potential: 
\begin{align*}
    \Phi(T, R(T)) - \Phi(0,0)
    = & \int_0^T \iprod{\nabla_{x} \Phi(t, R_t)}{\diff R(t)}
    + \int_{0}^T \partial_t \Phi(t, R(t)) \diff t\\
    &+ \frac{1}{2} \sum_{i,j \in [n]} \int_{0}^T \partial_{i j} \Phi(t, R(t)) \diff [R_i,R_j]_t.
\end{align*}
For the first term above, note that $\iprod{\nabla_{x} \Phi(t, R(t))}{\diff R(t)}
= \iprodt{\ones}{\nabla_{x} \Phi(t, R(y))} \cdot \iprod{p(t)}{\diff R(t)}$ by the definition of \(p(t)\) in~\eqref{eq:potential_player}. Furthermore, this term is zero since
\begin{equation*}
    \iprod{p(t)}{\diff R(t)} = \iprod{p(t)}{\diff G(t)} - \iprod{p(t)}{\ones} 
    \diff A(t) = \diff A(t) - \diff A(t) = 0.
\end{equation*}
Finally, it only remains to show that \(\diff [R_i,R_j]_t = (e_i - p(t))^{\transp} \Sigma(t) (e_j - p(t)) \diff t\) for all \(i,j \in [n]\) to conclude the proof of~\eqref{eq:generalPlayerPotFormula}.
By the definition of \(R_i(t)\), we have
\begin{equation*}
    \diff R_i(t)
    = \diff G_i(t) - \diff A(t)
    = \diff G_i(t) - \sum_{k = 1}^n p_k(t) \diff G_k(t)
    = \iprod{e_i - p(t)}{\diff G(t)}.
\end{equation*}
Moreover, using that \(\diff B_i(t) \cdot \diff B_j(t) = \boole{i = j} \diff t\), we have
\begin{equation*}
    \diff G_i(t) \cdot \diff G_j(t)
    = \iprod{w^{(i)}(t)}{\diff B(t)} \cdot \iprod{w^{(j)}(t)}{\diff B(t)} 
    = \iprod{w^{(i)}(t)}{w^{(j)}(t)} \diff t
    = \Sigma_{ij}(t) \diff t.
\end{equation*}
Combining these two facts we have
\begin{align*}
    \diff [R_i,R_j]_t 
    &= \diff R_i(t) \cdot \diff R_j(t) 
    = \iprod{e_i - p(t)}{\diff G(t)} \cdot
    \iprod{e_j - p(t)}{\diff G(t)}
    \\
    &= (e_i - p(t))^{\transp} \Sigma(t)
    (e_j - p(t)) \diff t,
\end{align*}
which concludes the proof of~\eqref{eq:generalPlayerPotFormula}

Let us now prove~\eqref{eq:PotLowerBound}. Suppose that 
 for all \(t \geq 0\) and \(x \in \Reals^n\) we have \(\partial_{ij} \Phi(t,x) = 0\) for all distinct \(i,j \in [n]\). 
Then,
\begin{equation*}
    \Phi(T, R_T) - \Phi(0,0)
    =  \int_{0}^{T}  \Bigg(\partial_t \Phi(t, R(t)) + \frac{1}{2}\sum_{i = 1}^n
    (e_i - p(t))^{\transp} \Sigma(t) (e_i - p(t)) \partial_{ii} \Phi(t, R(t)) \Bigg)\diff t.
\end{equation*}
Since \(\Sigma(t)\) is positive definite with ones in its diagonal entries, we have \(\abs{\Sigma_{ij}(t)} \leq \abs{\Sigma_{ii}(t)} = 1\). Therefore, for any \(v \in \Reals^n\), we have
$\iprodt{v}{\Sigma(t) v}
    %\leq \sum_{i = 1}^n \abs{v_i} \sum_{j = 1}^n \abs{v_j}
    \leq \norm{v}_1^2$.
Thus, if \(\partial_{ii}\Phi(t,x) \leq 0\) for all $i \in [n]$, then the second inequality stated in the lemma follows since \(\norm{e_i - p(t)}_1 \leq 2\) for all \(i \in [n]\).
\end{proof}

\begin{wrapfigure}{l}{0.42\textwidth}
    \centering
    \vspace{-27pt}
    \begin{mdframed}
        \captionsetup{format=plain}
        \includegraphics[width=\textwidth]{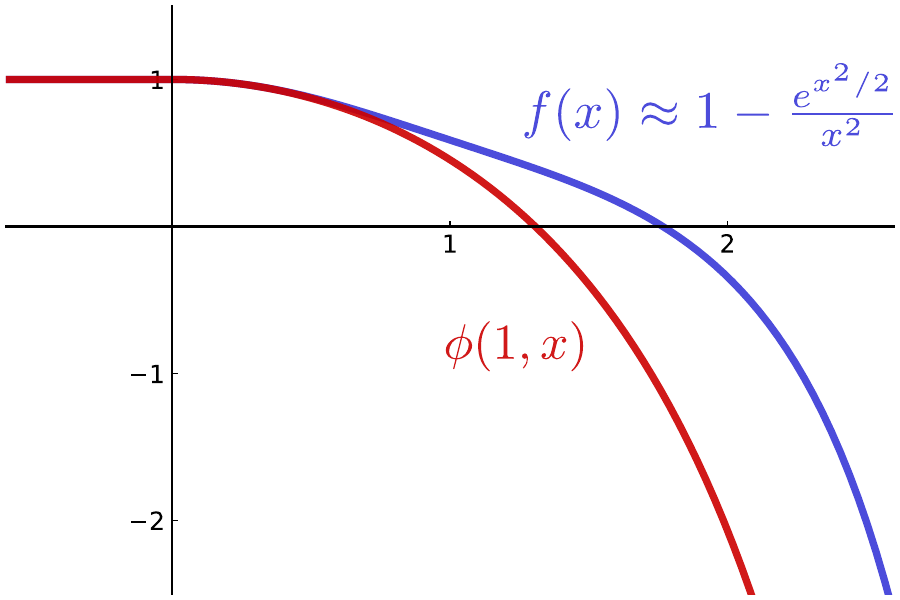}
    \caption{Plot of $\phi(1,x)$ in red and of the bound from \Lemma{expx2} in blue. }
    \label{fig:mo_plot}
    \end{mdframed}
\end{wrapfigure}
The second expression in \Lemma{RegretPlayerPotFormula} (\Equation{PotLowerBound}) hints at properties of potential functions \(\Phi\) that may be particularly useful. More precisely, separable functions \(\Phi\) that satisfy a diffusion constraint of the form \(\left(\partial_t + 2 \sum_{i = 1}^n \partial_{ii}\right) \Phi(t,\alpha) \geq 0\) would guarantee that \(\Phi(t, R(t))\) is non-decreasing in \(t\), which in turn may allow us to bound the continuous regret.

% \victor{\textbf{to do} Say that \(-\phi(1,x) \approx \exp(x^2/2)/x^2\). Probably add plot to ease the burden of understanding \(\phi\).}
The player strategies in the rest of this paper involve the function \(M_0\) defined as
\begin{equation*}
    M_0(\alpha) \coloneqq e^\alpha - \sqrt{\pi \alpha} \erfi(\sqrt{\alpha}), \qquad \forall \alpha \in \Reals.
\end{equation*}
This is an example of a \emph{confluent hypergeometric function} (of the first kind). We use $M_0$ in the form
\begin{equation}
    \label{eq:phi_m0_def}
    \phi(t, \alpha) \coloneqq \sqrt{t} M_0\paren[\Big]{\frac{[\alpha]_+^2}{2t}},
    \quad \forall \alpha \in \Reals, \forall t > 0.
\end{equation}
Similar functions have been used in the stochastic process literature \citep{Breiman,Davis76, Perkins} and in the online learning literature (see~\citealt[eq.~(2.6)]{HarveyLPR20a} and \citealt[eq.~(11)]{ZhangCP22}).
Two particularly useful properties of $\phi$ are:
\begin{itemize}
\item \((\partial_t + \frac{1}{2} \partial_{xx})\phi(t,\alpha)\) is zero\footnote{Actually, \(\phi\) is not doubly differentiable in its second argument because of the truncation in the definition of \(\phi\). Although this might seem like a problem to apply \Ito's formula, we luckily have a single point of non-differentiability at each time \(t \geq 0\), and the truncation makes \((\partial_t + \frac{1}{2}\partial_{xx})\phi(t,x)\) no smaller everywhere else. Thus, standard smooth-truncation arguments can be made to apply \Ito's formula. For an example, see~{\citet[Section~5.2.2]{HarveyLPR20a_arxiv}}. For the sake of simplicity, we set \(\partial_{xx} \phi(t,0) \coloneqq \lim_{\eps \to 0} \partial_{xx} \phi(t,\eps)\) } for all \(t > 0\) and $ \alpha \geq 0$, and non-negative for all \(\alpha < 0\). This is a PDE known as the \emph{backwards heat equation} (BHE). Diffusion terms like these appear in \Ito's formula, so functions satisfying the BHE are well-behaved under stochastic integration.

\item \(\phi(t,\alpha) \approx -\frac{1}{\alpha^2}e^{\alpha^2/2}\) (see~\Lemma{expx2} and Figure~\ref{fig:mo_plot}), and so the potential resembles the density of the normal distribution. 
Potentials of this form have been useful in the literature
such as for NormalHedge~\citep{ChaudhuriFH09} and
AdaNormalHedge~\citep{LuoS15}.
Moreover, \citet{Freund09a} has used these normal-like potentials in continuous time.
\end{itemize}

The algorithm in this section uses the separable potential function \(\Phi\) given by\footnote{Ideally we would like to modify the potential by eliminating the denominator of $2$.
If this could be analyzed, it would yield an optimal quantile regret bound. At present, we have been unable to accomplish this.}
\begin{equation}
    \label{eq:hyperg_pot}
      \Phi(t, x) \coloneqq \sum_{i = 1}^n \phi\paren[\Big]{t, \frac{x_i}{2}}
     \qquad \forall t > 0, \forall x \in \Reals^n.
\end{equation}

\begin{lemma}
    \LemmaName{CtsPotentialNonNegative}
    Let \(\Phi\) and \((p(t))_{t \geq 0}\) be as in~\eqref{eq:hyperg_pot} and~\eqref{eq:potential_player}, respectively. Then \(\Phi(T, R(T)) \geq 0\) for all~\(T \geq 0\) almost surely.
\end{lemma}
Note that if $\Phi(0, 0) = 0$ then \Lemma{CtsPotentialNonNegative} would immediately follow from \Lemma{RegretPlayerPotFormula} (in particular, \Equation{PotLowerBound}) and our choice of $\Phi$ since $\Phi$ is concave and $\left( \partial_t + 2\sum_{i=1}^n \partial_{ii} \right) \Phi(t, x) \geq 0$.
The only minor snag is that $\Phi(0, 0)$ is not well-defined since \Equation{phi_m0_def} would involve a division by zero.
Nonetheless, it is possible to resolve this issue since \(\lim_{\delta \downarrow 0} \Phi(\delta, 0) = 0\); the details are relegated to \Appendix{CtsPotentialNonNegative}.
% Nonetheless, one can easily resolve this issue by defining $\Phi(0, 0) \coloneqq \lim_{\eps \to 0} \Phi(\eps, 0) = 0$.
% We relegate the technical details to \Appendix{CtsPotentialNonNegative}.

% Since \(\phi\) is concave and satisfies \((\partial_t + \frac{1}{2} \partial_{xx}) \phi(t,x) \geq 0\), \Lemma{RegretPlayerPotFormula} tells us (modulo some care with \(t = 0\) that we discuss later) that \(\Phi(t, R(t)) \geq 0\) for all \(t > 0\).
% \chris{Do we need to be a bit careful here? I think we might have $\Phi(t, R(t)) \geq \Phi(0, R(0))$ but we might need to be careful about $0$.
% In the two-experts paper, we handled this by saying that we had $\Phi(t, R(t)) \geq \Phi(eps, R(\eps))$ for all $t \geq \eps$ and then taking $\eps \to 0$ while noting that
% $\lim_{\eps \to 0} \Phi(\eps, R(\eps)) = 0$.
% }
It remains to translate bounds on the value of the potential \(\Phi(t, R(t))\) to bounds on the  continuous regret. The following function is used in the regret bounds throughout the remainder of the paper.
\begin{definition}
\DefinitionName{lambda}
For $\alpha \in \bR_{\geq 0}$, let $\lambda(\alpha) > 0$ be the unique positive solution to the equation
\begin{equation}
    \EquationName{M0_unique_solution}
    \alpha = -M_0( \lambda(\alpha)^2/2 ) \equiv -\phi(1, \lambda(\alpha)).
\end{equation}
\end{definition}
We note that the function $M_0(x^2/2)$ is \emph{strictly} decreasing on $\bR_{\geq 0}$ and the image of $M_0$ on $\bR_{\geq 0}$ is $(-\infty, 1]$ (see~\Appendix{M0Properties}). 
In particular, a solution to \Equation{M0_unique_solution} exists and is unique so $\lambda(\alpha)$ is well-defined for all $\alpha \geq 0$.
We also note that $\lambda(\alpha)$ is strictly increasing in $\alpha$.

The next lemma show us how bounds in the value of \(\Phi(t, x)\) can be translated into bounds on the quantiles of $x$,
which were defined in \eqref{eq:QuantileDef}.
%Although these involve \(\lambda\), we can use the fact that \(\lambda(\alpha) \approx \sqrt{2 \ln \alpha}\) (see~\Appendix{QuantileRegretM0}) to compare how these bounds stack against other well-known bounds.
\begin{lemma}
    \LemmaName{M0ToRegret}
    Let \(T > 0\) and \(x \in \Reals^n\). Suppose that \(\Phi(T, x) \geq 0\). Then, for any \(\eps \in (0,1]\),
    \begin{equation*}
        \quantile(\eps,x)
        ~\leq~ 2\lambda\paren[\Big]{\frac{1 - \eps}{\eps}} \sqrt{T}
        ~\leq~ 2 \big(3 + \sqrt{2 \ln(1/\eps)}\big) \sqrt{T}.
    \end{equation*}
%    In particular, by taking \(\eps = 1/n\) we have
%    \begin{equation*}
%        \max_{i \in [n]} x_i \leq 2\lambda(n - 1) \sqrt{T} \leq 2\sqrt{2 \ln(n) T} + 8 \sqrt{T}
%    \end{equation*}
\end{lemma}
\begin{proof}
    For simplicity, suppose \(x_1 \geq x_2 \geq \dotsc \geq x_n\). Since \(\phi(T, \cdot)\) is decreasing, we have \(\phi(T, x_{\eps n}/2) \geq \phi(T, x_i/2)\) for every \(i \leq \eps n \). Summing this inequality for \(i \in \{1,\cdots,\eps n\}\), using the assumption \(\sum_{i =1}^n \phi(T, x_i/2) = \Phi(T,x) \geq 0\) and since \(-\phi(T,\alpha) \geq -\sqrt{T}\) for any \(\alpha \in \Reals\), we have
\begin{equation}
    \label{eq:lambda_regret_1}
    \eps n \phi(T,x_{\eps n}/2) \geq \sum_{i = 1}^{\eps n} \phi(T, x_i/2) \geq - \sum_{i = \eps n + 1}^n \phi(T,x_i/2) \geq - (1 - \eps)n\sqrt{T}.
\end{equation}
By the definition of \(\phi\), the above series of inequalities implies
\begin{equation*}
    \eps n \sqrt{T}M_0\paren[\Big]{\frac{([x_{\eps n}]_+/2)^2}{2T}} \geq - \paren[\big]{(1 - \eps)n}\sqrt{T}
    \implies 
    M_0\paren[\Big]{\frac{([x_{\eps n}]_+/2)^2}{2T}} \geq - \frac{1 - \eps}{\eps} \eqqcolon - \gamma.
\end{equation*}
Using the definition of \(\lambda\) and the fact that \(M_0\) is a decreasing function, we have
\begin{equation*}
    M_0\paren[\Big]{\frac{([x_{\eps n}]_+/2)^2}{2T}} \geq
    M_0\paren[\Big]{\frac{\lambda(\gamma)^2}{2}}
    \implies x_{\eps n} \leq 2\lambda(\gamma) \sqrt{T}.
\end{equation*}
The second inequality in~\eqref{eq:lambda_regret_1} follows from \Lemma{lambdabound}.
% First, note that \(\Phi(T, x) \geq - \beta \sqrt{T}\)
% Observe that $M_0([x_{\eps n}]_+^2 / 2t) \geq M_0([x_{i}]_+^2 / 2t)$ for all $i \leq \eps n$.
% Summing up both sides from $i = 1, \ldots, \eps n$, we have
% \begin{align*}
%     \eps n M_0\left( \frac{[R_{t,\eps n}]_+^2}{2t} \right)
%     & \geq \sum_{i=1}^{\eps n} M_0 \left( \frac{[R_{t,i}]_+^2}{2t} \right) \\
%     & \geq -\sum_{i > \eps n} M_0 \left( \frac{[R_{t,i}]_+^2}{2t} \right) \\
%     & \geq -(1-\eps) n,
% \end{align*}
% where the second inequality is by \Equation{potential} (and dividing out the $\sqrt{t}$) and the third inequality is because $M_0(x) \leq 1$ for $x \geq 0$.
% We conclude that $M_0\left( [R_{t,\eps n}]_+^2 / 2t \right) \geq -\frac{1-\eps}{\eps}$.
\end{proof}

Finally, we can combine \Lemma{CtsPotentialNonNegative} and \Lemma{M0ToRegret} to prove a bound on the quantile regret in continuous-time.
% The proof of the following theorem by combining Lemmas~\ref{lem:CtsPotentialNonNegative} and \ref{lem:M0ToRegret}.
% It gives anytime bounds on the continuous regret and quantile regret using the player strategy in~\eqref{eq:potential_player} with the potential \(\Phi\) defined in~\eqref{eq:hyperg_pot}.
%Although the bounds on the continuous regret are worse than the ones given by the  continuous MWU.
These quantile regret bounds improve upon the best known in the discrete case. In~\Section{DiscreteQuantileRegret} we discretize this algorithm while preserving the same quantile regret bound.

\begin{theorem}
    \TheoremName{LooseM0Regret}
    Let \(\Phi\) and \((p(t))_{t \geq 0}\) be as in~\eqref{eq:hyperg_pot} and~\eqref{eq:potential_player}, respectively. Then almost surely % \(\Phi(T, R(T)) \geq 0\) and
    %\(\ContRegret(T) \leq 2\lambda(n) \sqrt{T}\) and
    \[
    \QuantRegret(\eps, T) ~\leq~ 2\lambda((1 - \eps)/\eps)\sqrt{T}
    ~\leq~ 2 \big(3 + \sqrt{2 \ln(1/\eps)}\big) \sqrt{T}
        \quad \forall T \geq 0.
    \]
\end{theorem}

\subsection{Discretization}
\SectionName{DiscreteQuantileRegret}

In this section, we propose an algorithm for the original experts' problem based on the continuous-time solution of the previous section. As in the continuous setting, we have \(n \in \Naturals\) experts.
At each round \(t \in \Naturals\), the player picks a probability vector \(p_t \in \Delta_n\) and the adversary picks a gain vector~\(g_t \in [-1,1]^n\). The \textbf{instantaneous regret vector} at round \(t \geq 1\) is given by $r_t \coloneqq g_t - \ones \cdot \iprodt{p_t}{g_t}$.
% \begin{equation*}
%     r_{t} \coloneqq g_t - \ones \cdot \iprodt{p_t}{g_t}.
% \end{equation*}
Moreover, define the \textbf{regret vector} at round \(t\) by \(R_t \coloneqq \sum_{s = 1}^t r_s\).

To discretize the algorithm from the previous section, we shall make use of \emph{discrete derivatives} in a way similar to \citet{HarveyLPR20a_arxiv}.
% More precisely, define the discrete derivatives of \(\phi\)`' to be % with respect to the second argument by
For a bivariate function $f$, define its discrete derivatives as
\begin{equation}
\EquationName{DiscreteDerivDef}
\begin{aligned}
    f_t(t,x) & ~=~ f(t, x) - f(t-1, x), \\
    f_x(t,x) & ~=~ \frac{f(t, x+1)-f(t,x-1)}{2}, \\
    f_{xx}(t,x) & ~=~ (f(t, x+1) + f(t, x-1)) - 2f(t, x).
\end{aligned}
\end{equation}
Let \(\Phi\) be defined as in \Equation{hyperg_pot}.
For $i \in [n]$, we define the discrete derivative of $\Phi$ as
\begin{equation*}
\begin{aligned}
    \Phi_t(t,x)  ~=~ \sum_{i=1}^n \phi_{t} \paren[\Big]{t, \frac{x_i}{2}}; \quad
    \Phi_i(t,x)  ~=~ \frac{1}{2}\phi_x \paren[\Big]{t, \frac{x_i}{2}}; \quad
    \Phi_{ii}(t,x)  ~=~ \frac{1}{4}\phi_{xx} \paren[\Big]{t, \frac{x_i}{2}}.
\end{aligned}
\end{equation*}
% Note the use of the constant factors for the first and second derivative in the second argument to ``simulate'' the chain rule.
For notation convenience, we also define the discrete gradient $\dnabla \Phi(t,x) \coloneqq (\Phi_1(t,x), \ldots, \Phi_n(t,x))^{\transp}$.
% \chris{Fix the square.}
% $(\dnabla \Phi(t,x))_i \coloneqq \phi_x \paren[\Big]{t, \frac{x_i}{2}}$.
% We define the discrete gradient of $\PhiSo far $ as
% % The the algorithm for the discrete expert's problem we will use the following discrete approximation of the gradient of \(\Phi\):
% \begin{equation*}
%     (\dnabla \Phi(t,x))_i \coloneqq \phi_x\paren[\Big]{t, \frac{x_i}{2}}
%     \qquad \forall i \in [n].
% \end{equation*}

% For a strategy $p_t$ at time $t$, the algorithm's cost is $\mu_t = \sum_{i} p_{t,i} c_{t,i}$.
% Let $r_{t,i} = \mu_t - c_{t,i}$ be the change in regret at time $t$ with respect to expert $i$ and let $R_t = \sum_{s \leq t} r_s$
% where $r_s = (r_{s,1}, \ldots, r_{s,n})$ is the vector of changes in regret at time $s$.

The algorithm we use for the discrete setting is the natural analogue of the algorithm for the continuous setting as defined in \Equation{potential_player}.
Specifically, for $t \in \Naturals_{\geq 1}$, we set
\begin{equation}
    \label{eq:discrete_quantile_algo}
    p_t \coloneqq
    \begin{cases}
        \frac{1}{n} \ones &\quad\text{if $\dnabla \Phi(t, R_{t-1}) = 0$} \\
        \frac{1}{\iprodt{\ones}{\dnabla \Phi(t,R_{t-1})}}\dnabla \Phi(t, R_{t-1}) &\quad\text{otherwise.}
    \end{cases}
\end{equation}
Note that $\phi(t, x)$ is non-increasing in $x$ so $\dnabla \Phi(t, x) \leq 0$ (component-wise), which implies $p_t \in \Delta_n$. Let us now analyze the performance of this algorithm. We summarize our results in the next theorem.

% \begin{itemize}
%     \item If $\phi_x(t,R_{t,i}) = 0$ for all $i$, play arbitrarily.
%     (This is the case when all the regrets are at most $-1$.
%     Note that $\phi$ is a non-decreasing function so $\phi_x$ is always non-negative.)
%     \item Otherwise, play $p_{t,i} \propto \phi_x(t+1, R_{t,i})$, i.e.
%     \[
%         p_{t,i} = \frac{\phi_x(t+1, R_{t,i})}{\sum_j \phi_x(t+1, R_{t,j})}.
%     \]
% \end{itemize}

%In particular, we prove the following theorem.
\begin{theorem}
\TheoremName{eps_quantile_regret}
We have \(\quantile(\eps, R_t) \leq 2\lambda((1-\eps)/\eps) \sqrt{t} \leq 2\sqrt{t} + \sqrt{8\ln(1/\eps) + 20} \cdot \sqrt{t}\).
\end{theorem}
Let $c$ be the optimal constant multiplying  $\sqrt{t \ln(1/\eps)}$ in the minimax optimal \(\eps\)-quantile regret for anytime algorithms.
\Theorem{eps_quantile_regret} shows $c \leq 2\sqrt{2}$.
Previously, \citet[Theorem 3]{ChernovV10} and \citet[Example 3]{NegreaBCOR21} both proved\footnote{Note that these papers consider costs in $[0,1]$ so their results must be multiplied by $2$ to coincide with our setting.} $c \leq 4$.
On the lower bound side, \citet[Theorem 1]{NegreaBCOR21} proved that $c \geq \sqrt{2}$.
So there remains a gap of $2$ for the constant $c$.
Finally, we note that if $T$ is known beforehand (the fixed-time setting), \citet[Corollary 6]{OrabonaP16} showed $c \leq 2 \sqrt{3}$.
\Theorem{eps_quantile_regret} improves this to $c \leq 2\sqrt{2}$.

Interestingly, \citet{ZhangCP22} proposed independently of us an algorithm using coin-betting with a potential similar to \eqref{eq:hyperg_pot} also inspired by the work of~\citet{HarveyLPR20a}. Although their work is mostly on unconstrained online learning, one can obtain an algorithm for the experts' problem from their coin-betting algorithm whose \(\eps\)-quantile regret is similar to our bound from~\Theorem{eps_quantile_regret}, that is, 
roughly no more than \(2\sqrt{2 t \ln(1/\eps)}\). More precisely, \citet[Theorem~4]{OrabonaP16} show how to obtain an algorithm for the experts' problem with costs in \([0,1]\) from a coin-betting algorithm together with regret guarantees against any comparison point \(u \in \Delta_n\). We may obtain bounds on the \(\eps\)-quantile regret by noting that it is the same as the regret against all points \(u \in \Delta_n\) with \(\lceil \eps n \rceil\) non-zero equal entries. Combining this with Theorem 1 and Lemma~B.2 of~\citet{ZhangCP22} yields the desired bound.
In contrast, our analysis in this section is tailored specifically for the experts problem, and does not make use of the coin-betting framework.

To prove \Theorem{eps_quantile_regret}, we make use of the following lemma.
\begin{lemma}
    \LemmaName{LBIsDiscreteIto}
    Let $x_1, x_2, \cdots$ be a sequence of real numbers such that $|x_t - x_{t-1}| \leq 1$.
    Then for any function $f$ that is \emph{concave} in its second argument and any integer $T \geq 2$, we have
    \begin{equation*}
    \begin{aligned}
        f(T, x_T) - f(1, x_1)
        ~\geq~\sum_{t=2}^T f_x(t, x_{t-1})\cdot (x_t - x_{t-1})
        ~+~\sum_{t=2}^T \left( \frac{1}{2} f_{xx}(t, x_{t-1}) + f_t(t, x_{t-1}) \right).
    \end{aligned}
    \end{equation*}
\end{lemma}

The above lemma is essentially a corollary of the following discrete version of \Ito's formula (see \citealp[Example 10.9]{Klenke} or \citealp[Lemma~3.13]{HarveyLPR20a_arxiv}).
\begin{lemma}[Discrete \Ito's Formula]
    \LemmaName{AlmostDiscreteIto}
    Let $x_1, \cdots$ be a sequence of real numbers.
    Then for any function $f$ and any fixed time $T \geq 2$, we have
    \begin{equation}
    \begin{aligned}
        f(T, x_T) - f(1, x_1)
        & = \sum_{t=2}^T f(t, x_t) - \frac{f(t, x_{t-1} + 1) + f(t, x_{t-1} -1)}{2} \\
        & + \sum_{t=2}^T \left( \frac{1}{2} f_{xx}(t, x_{t-1}) + f_t(t, x_{t-1}) \right).
    \end{aligned}
    \end{equation}
\end{lemma}
\begin{proofof}[\Lemma{LBIsDiscreteIto}]
   We prove the following statement. Let $f$ be a bivariate function that is concave in its second argument.
   Then for all $t, x \in \bR$ and $y \in [-1, 1]$ we have
   \[
        f(t, x+y) - \frac{f(t,x+1)-f(t,x-1)}{2} \geq f_x(t, x) \cdot y. 
   \]
   Equality holds for $y \in \{-1, 1\}$. Since the LHS is concave in $y$ and the RHS is linear in $y$,
   the inequality holds for all $y \in [-1, 1]$.
   The lemma now follows from \Lemma{AlmostDiscreteIto} with $x = x_{t-1}$ and~$y = x_t - x_{t-1}$.
\end{proofof}

%%%%%
We now prove a discrete analogue of the second assertion of \Lemma{RegretPlayerPotFormula}.
Note that for technical reasons we start at $t = 1$ instead of $t = 0$.
One can directly deal with the case $t = 0$ by customizing \Equation{phi_m0_def}. However, this cumbersome approach does not yield improved bounds.
\begin{lemma}
    \LemmaName{DiscreteRegretPlayerPotFormula}
    Fix any $T \geq 2$. Then
    \begin{align*}
        \Phi(T, R_T) - \Phi(1, R_1)
        & ~\geq~ \sum_{t=2}^T \left( \Phi_t(t, R_{t-1}) + 2 \sum_{i=1}^n \Phi_{ii}(t, R_{t-1}) \right) \\
        & ~=~ \sum_{t=2}^T \sum_{i=1}^n \left( \phi_t(t, R_{t-1,i}/2) + \frac{1}{2} \phi_{xx}(t, R_{t-1,i}/2) \right).
    \end{align*}
\end{lemma}
\begin{proof}
    Note that the equality follows from the definition of $\Phi_{ii}$ and $\Phi_t$.
    Here, we prove that the term on the left is an upper bound on the final term on the right.

    Recall that $\Phi(t, R_t) = \sum_{i=1}^n \phi(t, R_{t,i}/2)$.
    Since the gains are in $[-1, 1]$, it follows that $|(R_{t,i} - R_{t-1,i})/2| \leq 1$.
    From \Lemma{LBIsDiscreteIto}, since $\phi$ is concave in its second argument, we have
    \begin{equation}
    \EquationName{phiLB}
    \begin{aligned}
        \sum_{i=1}^n\paren[\big]{ \phi(T, R_{t,i}/2) - \phi(1, R_{1,i}/2)}
        & ~\geq~ \sum_{t=2}^T \sum_{i=1}^n \phi_x(t, R_{t-1,i}/2) \cdot (R_{t,i} - R_{t-1,i}) \\
        & ~+~ \sum_{t=2}^T \sum_{i=1}^n \left( \frac{1}{2} \phi_{xx}(t, R_{t-1,i}/2) + \phi_t(t, R_{t-1,i}/2) \right).
    \end{aligned}
    \end{equation}
    To prove the lemma, it suffices to show that the first sum on the RHS of \Equation{phiLB} is exactly $0$.
    To that end, fix $t \in [T]$ such that $t \neq 1$.
    We have $R_{t,i} - R_{t-1,i} = g_{t,i} - p_t^\transp g_t$.
    Hence,
    \begin{equation}
        \EquationName{phiLBFirstTerm}
        \sum_{i=1}^n \phi_x(t, R_{t-1,i}/2) \cdot (R_{t,i} - R_{t-1,i})
        ~=~ \sum_{i=1}^n \phi_x(t, R_{t,i}/2) \cdot (g_{t,i} - p_t\transpose g_t).
    \end{equation}
    % We claim that \Equation{phiLBFirstTerm} is non-negative.
    If $\phi_x(t, R_{t-1,i}/2) = 0 ~\forall i \in [n]$ then the RHS of \eqref{eq:phiLBFirstTerm} is $0$.
    Otherwise, with $p_t$ as defined in \eqref{eq:discrete_quantile_algo},
    \begin{equation}
        \EquationName{pt_bound}
        p_t\transpose g_t
        = \frac{\sum_{i=1}^n \phi_x(t, R_{t-1,i}/2) \cdot g_{t,i}}{\sum_{i=1}^n \phi_x(t, R_{t-1,i}/2)}.
    \end{equation}
    Plugging \eqref{eq:pt_bound} into \eqref{eq:phiLBFirstTerm} gives $\sum_{i=1}^n \phi_x(t, R_{t-1,i}/2) \cdot (R_{t,i} - R_{t-1,i}) = 0$, as required.
\end{proof}
In the continuous setting, we used the key fact that for any $t \in \bR_{\geq 0}$ and $x \in \bR$,
we have $\left( \partial_t + \frac{1}{2} \partial_{xx} \right) \phi(t, x) \geq 0$.
Luckily, the same fact holds in the discrete setting.
\begin{lemma}
    \LemmaName{DiscreteBHE}
    For any $t > 1$ and $x \in \bR$, we have $\phi_t(t,x) + \frac{1}{2} \phi_{xx}(t,x) \geq 0$.
\end{lemma}

Before we prove the above lemma, let us combine it together with the results in this section to prove \Theorem{eps_quantile_regret}.

\begin{proofof}[\Theorem{eps_quantile_regret}]
   \Lemma{DiscreteRegretPlayerPotFormula} and \Lemma{DiscreteBHE} imply that $\Phi(T, R_T) \geq \Phi(1, R_1)$ for all $T \geq 1$.
   Note that $\Phi(1, R_1) = \sum_{i=1}^n \phi(1, R_{1,i}/2) \geq n \cdot M_0(1/2) > n\cdot M_0(\gamma^2/2) = 0$.
   The bound on the quantile regret now follows from \Lemma{M0ToRegret}.
\end{proofof}

Finally, let us conclude this section with the proof of \Lemma{DiscreteBHE}.

\begin{proofof}[\Lemma{DiscreteBHE}]
    Recalling the definition of the discrete derivatives (see \Equation{DiscreteDerivDef}),
    we have that
    \[
        \phi_t(t,x) + \frac{1}{2} \phi_{xx}(t,x)
        ~=~ \frac{\phi(t, x+1) + \phi(t, x-1)}{2} - \phi(t-1, x).
    \]
    Hence, it suffices to prove that
    \begin{equation}
        \EquationName{DiscreteBHESufficient}
        \phi(t, x+1) + \phi(t, x-1) \geq 2 \phi(t-1, x).
    \end{equation}
    We consider several cases depending on the value of $x$.
    
    \paragraph{Case 1: $x \geq 1$.}
    In this case, \Equation{DiscreteBHESufficient} is equivalent to
    \begin{equation}
        \EquationName{DiscreteBHECase1}
        \sqrt{t} \cdot M_0\left( \frac{(x+1)^2}{2t} \right)
        ~+~ \sqrt{t} \cdot M_0\left( \frac{(x-1)^2}{2t} \right)
        ~\geq~ 2\sqrt{t-1} \cdot M_0 \left( \frac{x^2}{2(t-1)} \right).
    \end{equation}
    Note that $t > 1$ so all terms are well-defined.
    Rearranging, this is equivalent to
    \[
        M_0\left( \frac{(x+1)^2}{2t} \right)
        ~+~ M_0 \left( \frac{(x-1)^2}{2t} \right)
        ~\geq~ 2\sqrt{1-1/t} \cdot M_0 \left( \frac{x^2}{2t} \right)
    \]
    which follows from \Lemma{M0_ineq} by setting $x$ and $z$ in \Lemma{M0_ineq} to $x/\sqrt{t}$ and $1/\sqrt{t}$, respectively.
    
    \paragraph{Case 2: $x \in [0,1]$.}
    Note that \Equation{DiscreteBHECase1} holds for all $x \in \bR$ (because \Lemma{M0_ineq} holds for all $x \in \bR$).
    Next, observe that $\phi(t,x) \geq M_0(x^2/2t)$ with equality whenever $x \geq 0$ (this is because $M_0(x^2/2t)$ is increasing on the interval $(-\infty, 0]$ while $\phi(t,x) = \phi(t,0)$ due to the truncation defined in \Equation{phi_m0_def}).
    So the LHS of \Equation{DiscreteBHECase1} is upper bounded by the LHS of \Equation{DiscreteBHESufficient} and the RHS of \Equation{DiscreteBHECase1} is equal to the RHS of \Equation{DiscreteBHESufficient}.
    So \Equation{DiscreteBHECase1} holds in this case as well.
    
    \paragraph{Case 3: $x \in [-1, 0]$.}
    Note that $\phi(t,x+1)$ is the only term in \Equation{DiscreteBHESufficient} that is not constant on the interval $[-1, 0]$.
    Further, note that \Equation{DiscreteBHESufficient} holds for $x = 0$ by the previous case and it holds for $x = -1$ because the LHS is $2\sqrt{t}$ and the RHS is $2\sqrt{t-1}$.
    So the inequality holds since $\phi$ is concave in its second argument (\Lemma{PropertiesOfPhi}).
    
    \paragraph{Case 4: $x \leq -1$.}
    Due to the truncation in \Equation{phi_m0_def}, \Equation{DiscreteBHESufficient} becomes $2\sqrt{t} \geq 2\sqrt{t-1}$.
\end{proofof}

\vspace*{-14pt}

\section{Minimax Optimal Continuous Regret with Independent Experts}
\SectionName{IndependentExperts}

We have shown that an algorithm based on the hypergeometric potential~\eqref{eq:hyperg_pot} suffers no more than \(2 \sqrt{2 t \ln n}\) at any round \(t\) (\Theorem{LooseM0Regret}). Similarly,  we saw that  the \emph{anytime} continuous MWU suffers at most \(2\sqrt{t \ln n}\) regret at any round \(t\) (\Theorem{ContinuousMWURegret}) while its fixed-time variant suffer no more than \(\sqrt{2t \ln n}\). A natural question is whether there are anytime algorithms in the continuous setting that enjoy regret almost surely smaller than \(2\sqrt{t \ln n}\). 

In the discrete version of the problem, we know that no algorithm can guarantee regret smaller than \(\sqrt{2 t \ln n}\), although there is no known anytime algorithm that attains the lower-bound. This lower-bound comes from the fact that the expected regret against an adversary that assigns \(\pm 1\) gains uniformly at random---or \emph{uniformly random adversary}, for short---on the experts gets arbitrarily close to \(\sqrt{2 t \ln n}\) as \(n\) grows, regardless of the player's strategy. Interestingly, no tighter lower-bounds are known for \emph{anytime} algorithms, that is, algorithms that do not know the length of the game. Furthermore, for two experts (\(n = 2\)) we know that \emph{there is} a separation between the minimax optimal regret in the fixed-time and anytime settings, and both lower bounds arise from the expected regret against the uniformly random adversary! Namely, for fixed-time algorithms the best possible regret is \(\sqrt{\frac{2T}{\pi}}\)~\citep{Cover67a} while for anytime algorithms the best possible regret is~\(\lambda(0)\sqrt{T}\)~\citep{HarveyLPR20a}, where \(\lambda(0) \approx 1.3069 > \sqrt{2/\pi} \approx 0.798\). In both cases the lower-bond arises from the expected regret against the uniformly random adversary (the latter with a suitably chosen stopping-time). Yet, to this day we do not know if the minimax regret attained by fixed-time and anytime algorithms are different as \(n\) tends to infinity.

% At the same time, there is no known anytime algorithm for the discrete problem that enjoys regret bounds better than \(2\sqrt{t \ln n}\), achieved by anytime MWU~\citep[Theorem~2.4]{Bubeck11a} 

Intriguingly, we show that in continuous time with \emph{independent experts}---that is, when each \((G_i(t))_{t \geq 0}\) is an independent Brownian motion---there is an anytime algorithm whose regret is at most \(\lambda(3n) \sqrt{t} \approx \sqrt{2t \ln n} \) for all \(t \geq 0\). Therefore, if there is a gap between the best regret attained by fixed-time and anytime algorithms, this shows that the anytime lower-bound does not arise from independent experts. Furthermore, in \Proposition{ContinuousLB} we show a matching lower-bound, just as in the discrete-time case.
We conjecture that this algorithm can be discretized, and $\sqrt{2t \ln n}$ anytime regret bound that holds with high probability against independent experts is possible in discrete time.

% One should note that independent experts is the source of a regret lower-bound not only for \(n\) arbitrarily large, but also for the case~\(n = 2\). More precisely, for 2 experts the expected regret with uniformly random \(\pm 1\) gains is \(\sqrt{\frac{2T}{\pi}}\), which is optimal~\citep{Cover67a}. Moreover, \citet{HarveyLPR20a_arxiv} showed the expected regret in the anytime setting (that is, with a unknown worst-case stopping time) against this random adversary is \(\lambda(0) \sqrt{T}\). This suggests that the case with independent experts seem to be the worst-case scenario. Achieving an anytime regret bound for this case that matches the lower-bound may be an indication that, for \(n\) and \(T\) large enough, the anytime experts' problem is no harder than the fixed-time case.

\subsection{Algorithm}
The player strategy we use in this section is similar to the one from the last section, with a crucial difference on \(\Phi\): we do not divide the second argument by 2. Namely, define
\begin{equation}
    \label{eq:hyperg_pot_2}
      \Phi(t, x) \coloneqq \sum_{i = 1}^n \phi\paren{t, x_i}
     \qquad \forall t \geq 0, \forall x \in \Reals^n,
\end{equation}
and we set \((p(t))_{t \geq 0}\) as in~\eqref{eq:potential_player} but using the above potential.
We can still use \Lemma{RegretPlayerPotFormula} to get a formula for \(\Phi(t, R(t))\). However, the lower-bound given in~\eqref{eq:PotLowerBound} is not of much use anymore since we do not have \((\partial_t + 2 \sum_{i = 1}^n\partial_{ii})\Phi(t,x) \geq 0\) anymore. Thus, we need to directly analyze the formula in~\eqref{eq:generalPlayerPotFormula}. It turns out that when the instantaneous correlation matrix \(\Sigma(t)\) is the identity matrix for all \(t \geq 0\), which happens when we have independent experts, we can show that this term does not become too negative. This term will be denoted
\begin{equation*}
    \sBHT(t) \coloneqq \partial_t \Phi(t,R(t)) + \frac{1}{2}\sum_{i = 1}^n \partial_{ii} \Phi(t,R(t)) \iprodt{(e_i - p(t))}{(e_i - p(t))}.
\end{equation*}

 That is, the sBHT is the integrand that appears in \Lemma{RegretPlayerPotFormula}. Intuitively, since \(\Phi\) satisfies the backwards-heat inequality, we should expect \(\sBHT(t)\) to not be too negative. That is exactly what we show in the next lemma.
 A complete proof is in \Section{CoolConjecture},
 but we give a brief sketch here.
\begin{lemma}
    \LemmaName{CoolConjecture}
    Let \(\Phi\) be as in~\eqref{eq:hyperg_pot_2}. Then \(\sBHT(t) \geq (2 - n)/2\sqrt{t}\) for all \(t > 0\).
\end{lemma}
\begin{proofsketch}
    For simplicity, fix \(t > 0\) and assume \(R_1(t) \geq R_2(t) \geq \dotsm \geq R_n(t)\). Moreover, explicitly write the dependency of the sBHT on the regret vector by writing \(\sBHT(t, R(t))\). The first step is to show that forcefully setting $R_n(t)$ to zero, denote such a vector by \(\tilde{R}(t)\), can only decrease the sBHT. 
That is, \(\sBHT(t, R(t)) \geq \sBHT(t, \tilde{R}(t))\). 
Then, we show that \(\sBHT(t, \tilde{R}(t)) + 1/2\sqrt{t}\) is equal to the sBHT restricted to the experts \(1, \dotsc, n - 1\). Then, by induction we get that \(\sBHT(t, R(t)) + (n -1)/2\sqrt{t}\) is greater than the sBHT for a single expert, which one can verify that is at least \(1/2\sqrt{t}\), concluding the proof.
\end{proofsketch}

% Thus, to analyze the evolution of the potential we can focus our analysis on the \textbf{skewed backwards-heat term} (of \(\Phi\))
% \begin{equation*}
%     \sBHT(t,x) \coloneqq \partial_t \Phi(t,x) + \frac{1}{2}\sum_{i, j \in [n]}
%     (e_i - p_t)^T \Sigma (e_j - p_t) \partial_{x_i x_i} \Phi(t, x).
% \end{equation*}

Finally, the above lemma together with the expression for \(\Phi(t, R(t))\) given by~\eqref{eq:generalPlayerPotFormula} yields the desired regret bound when the instantaneous covariance matrix \(\Sigma(t)\) is always the identity matrix.
\begin{theorem}
    \TheoremName{CtsOptRegret}
    Let \(\Phi\) be defined as in~\eqref{eq:hyperg_pot_2} and let \((p(t))_{t \geq 0}\) be as in \eqref{eq:potential_player}. Suppose \(\Sigma(t) = I ~\forall t \geq 0\). Then, almost surely for all \(t \geq 0\) we have \(\ContRegret(t) \leq \lambda(3n -1)\sqrt{t} \leq  \sqrt{2 t \ln n} + 6\sqrt{t}\).
\end{theorem}
\begin{fakeproof}
    Let \(T \geq 0\). By Lemma~\ref{lem:RegretPlayerPotFormula} we have\footnote{In fact, we should be careful with \(\Phi(0,0)\) as we were in the proof of  \Theorem{LooseM0Regret} in \Appendix{QuantileRegretM0}. Since the exact same trick works in this case as well, we take \(\Phi(0,0) \coloneqq 0\) for the sake of simplicity.}
    \begin{equation}
        \label{eq:indep_exp_1}
        \Phi(T, R(T)) 
        = \int_0^T \paren[\big]{\partial_t\Phi(t, R(t)) + \frac{1}{2}\sum_{i,j \in [n]} \partial_{ij} \Phi(t, R(t)) (e_i - p(t))^\transp \Sigma(t) (e_j - p(t))} \diff t.
    \end{equation}
    Since \(\Sigma(t) = I\) for all \(t \geq 0\), the above integrand is exactly \(\sBHT(t)\). Therefore, by \Lemma{CoolConjecture},
    \begin{equation*}
        \eqref{eq:indep_exp_1} = \int_0^T \sBHT(t) \diff t
        \geq \int_0^T \frac{2 - n}{2\sqrt{t}} \diff t
        = (2 - n)\sqrt{T} \geq -n \sqrt{T}.
    \end{equation*}
    Finally, we translate this lower bound $\Phi(T, R(T))$ to an upper bound on $\ContRegret(T)$.
    By an easy modification of \Lemma{M0ToRegret} for \(\eps = 1/n\) and non-zero lower-bounds on \(\Phi(T, R(T))\), we get
    \begin{equation*}
        \ContRegret(T) \leq \lambda(2n -1)\sqrt{T} \leq \sqrt{2t \ln(2n)} + 4\sqrt{T} \leq \sqrt{2 t \ln n} + 6\sqrt{T}. \qedhere
    \end{equation*}
\end{fakeproof}

\vspace*{-14pt}

\subsection{A matching lower-bound}
The next proposition shows that, against independent experts, the \emph{expected} regret is always roughly $\sqrt{2t \ln n}$.
This matches the discrete-time lower bound and shows that \Theorem{CtsOptRegret} is tight.
% In discrete time we know the expected regret against the uniformly random adversary asymptotically approaches \(\sqrt{2T \ln n}\), the best worst-case lower-bound known for general \(n\). In the next proposition we show an analogous result for the continuous-time setting.
The proof is a straightforward modification of the analogous discrete-time result \cite[Theorem 3.7]{Cesa-BianchiL06a}.
It is important to note that \Theorem{CtsOptRegret} is considerably stronger than \Proposition{ContinuousLB}: the latter bounds the expectation separately for each $t$, whereas the former gives a bound that holds almost surely for all $t$.

\begin{proposition}
    \PropositionName{ContinuousLB}
    Assume \(w^{(i)} = e_i\;\) for each \(i \in [n]\). Then, for any player strategy,
    \begin{equation*}
        \sqrt{2 t \ln n}(1 - o(1)) ~\leq~ \expect{\ContRegret(t)} ~\leq~ \sqrt{2 t \ln n}
        \qquad\forall t>0.
    \end{equation*}
\end{proposition}
\begin{proof}
    Since the functions in the stochastic differential equations in \Section{ContinuousExperts} are at least left-continuous and bounded, the stochastic integrals are well-defined, vanish at time \(0\), and are martingales~\citep[Def.~IV.2.1 and~IV.2.3]{RevuzY99a}. In particular, \(\expect{A(t)} = 0\) for all \(t \geq 0\).  For each \(i \in [n]\) we have \(G_i(t) = B_i(t)\) since \(w^{(i)}(t) = e_i\). Thus, each gain process is an independent Brownian motion. Therefore,
    \begin{equation*}
        \expect{\ContRegret(t)} = \expect{\max_{i \in [n]} B_i(t)} = \sqrt{2 T \ln n}(1 - o(1)),
    \end{equation*}
    where in the last equation we used the well-known asymptotics for the maximum of \(n\) Gaussian random variables (e.g.~\citealp[Exercise~2.11]{Wainwright19} or~\citealp[Theorem~3]{OrabonaP15a}) and the fact that for any \(i \in [n]\) the random variable \(B_i(t)\) follows a Gaussian distribution with mean zero and variance \(t\).
\end{proof}

%%%%%%%%%%%%%%%%%%%%%%%%%%%
\subsection{\protect Proof of \Lemma{CoolConjecture}}
\SectionName{CoolConjecture}

We will prove the lower-bound on the sBHT evaluated at any point \(x \in \bR^n\) instead of only when it is evaluated at \(R(t)\). Thus, we shall define some notation to analyze the sBHT evaluated at arbitrary points. Namely, let \(x \in \Reals^n\). Define
\begin{equation*}
    p(t,x) \coloneqq
    \frac{1}{\iprodt{\ones}{\nabla_{x} \Phi(t, x)}} \nabla_{x} \Phi(t, x),
     \qquad \forall t \geq 0,
\end{equation*}
setting \(p(t,x) = (1/n) \ones\) if $\iprodt{\ones}{\nabla_{x} \Phi(t, x)} = 0$. Moreover, define
\begin{equation*}
    \sBHT(t,x) \coloneqq \partial_t \Phi(t,x) + \frac{1}{2}\sum_{i = 1}^n \partial_{ii} \Phi(t,x) \iprodt{(e_i - p(t,x))}{(e_i - p(t,x))}.
\end{equation*}
Note now that we may assume \(x \geq 0\). Indeed, assume \(x_i < 0\) for some \(x_i\) and take \(t > 0\). Then, due to the truncation in the definition of \(\phi\), we have \(\partial_{ii} \Phi(t,x) = \partial_{xx} \phi(t,x_i) = 0\). Since \(\partial_t \phi(t,x_i) = \frac{1}{\sqrt{2t}} \exp(\frac{x_i^2}{2t}) \geq 0\), we conclude that \(\sBHT(t, x) \geq \sBHT(t, x - e_i x_i)\), that is, setting the \(i\)-th entry of \(x\) to zero can only decrease the value of the \(\sBHT\).

Let \(T > 0\). For each \(i \in [n]\) define
\begin{align*}
    q_i &\coloneqq e^{x_i^2/2T},&  Q_i&\coloneqq \erfi\left(\frac{x_i}{\sqrt{2T}}\right),\\
    \Theta &\coloneqq \sum_{j = 1}^n Q_j,&  p_i&\coloneqq \frac{Q_i}{\Theta} \equiv p(T, x)_i.
\end{align*}
Then, evaluating the derivatives according to \Lemma{PropertiesOfPhi} and using that \(p_i \propto \partial_x \phi(t, x_i) \propto Q_i\) we have
\begin{align*}
    \sBHT(T,x)
    &= 
    \frac{1}{2\sqrt{T}}\sum_{i = 1}^n e^{x_i^2/2t}\left(1 - (e_i - p)^{\transp} (e_i - p)\right)
    \\
    &= 
    \frac{1}{2\sqrt{T}}\sum_{i = 1}^n e^{x_i^2/2t}\left(2 p_i - p^{\transp} p\right)
    \\
    &= 
    \frac{1}{2 \Theta^2 \sqrt{T}}\sum_{i = 1}^n q_i\left(2 \Theta Q_i - Q^\transp Q\right).
\end{align*}
Let us now make a couple of simplifying assumptions:
\begin{itemize}
    \item Assume \(x_1 \geq x_2 \geq \dotsm \geq x_n\);
    \item We may assume without loss of generality that \(T = 1\) since, from the above calculations, one can see that \(\sBHT(T, x) = \sqrt{T}\sBHT(1, x/\sqrt{T})\).
\end{itemize}
Now, it suffices to show that
\begin{equation}
    \label{eq:derivative_claim}
    2 \partial_{x_n} (\Theta^2 \sBHT(1,x)) \geq 0  
\end{equation}
    to prove the desired claim by induction. Indeed, note that if \(\partial_{x_n} (\Theta^2 \sBHT(1,x)) \geq 0\), then decreasing \(x_{n}\) all the way to zero only decreases the value of \(\Theta^2 \sBHT(1,x)\). More specifically, let \(x' \coloneqq x - x_n e_n\) and define \(\Theta', Q'\), and $q'$ accordingly (that is, substituting \(x\) by \(x'\) in the original definition of these terms). In this case, we have \(Q_n' = 0\) and \(q_n' = 1\), yielding
    \begin{align*}
        \Theta^2 \sBHT(1, x)
        &\geq (\Theta')^2 \sBHT(1, x')
        = \frac{1}{2}\sum_{i = 1}^n q_i' (2 \Theta Q_i' - \iprodt{Q'}{Q'})\\
        &= \frac{1}{2}\sum_{i < n} q_i' (2 \Theta Q_i' - \iprodt{Q'}{Q'})
        - \frac{\iprodt{Q'}{Q'}}{2}.
    \end{align*}
    Dividing everything by \((\Theta')^2\), we get
    \begin{align*}
    \left(\frac{\Theta^{}}{\Theta^\prime}\right)^2 \sBHT(1, x)
    &\geq \frac{1}{2 (\Theta')^2}\sum_{i < n} q_i' (2 \Theta Q_i' - \iprodt{Q'}{Q'})
    - \frac{\iprodt{p'}{p'}}{2}
    \\
    &\geq \frac{1}{2(\Theta')^2}\sum_{i < n} q_i' (2 \Theta Q_i' - \iprodt{Q'}{Q'})
    - \frac{1}{2}.
    \end{align*}
    Note that if \(\sBHT(1,x) \geq 0\), then the lower-bound we are trying to prove in the statement of \Lemma{CoolConjecture} holds trivially, thus assume otherwise. Then \(\sBHT(1,x) \geq \left(\frac{\Theta^{}}{\Theta^\prime}\right)^2 \sBHT(1, x)\) since \(\Theta \geq \Theta'\). Finally, the last summation in the above calculation is exactly the sBHT with \(n-1\) experts at time \(1\). In the case of only 1 expert, one can easily check that the sBHT is always at least \(1/2\). So by induction we have \(\sBHT(1,x) \geq (2 - n)/2\), as we wanted to show. Thus, let us finally proceed with the proof of~\eqref{eq:derivative_claim}.

First, let us summarize the properties on the partial derivatives:
\begin{align*}
    \partial_{x_n} Q_n &= \sqrt{\frac{2}{\pi}} q_n,
    & \partial_{x_n} q_n &= x_n q_n, \\
    \partial_{x_n} \Theta &= \partial_{x_n} Q_n = \sqrt{\frac{2}{\pi}} q_n,
    & \partial_{x_n} (Q^{\transp}Q) &= 2 \sqrt{\frac{2}{\pi}} Q_n q_n.
\end{align*}
Then,
\begin{align}
    &2 \partial x_n(\Theta^2 \sBHT(1,x)) \nonumber
    \\
    =& \sum_{i = 1}^n \partial_{x_n}(q_i(2 Q_i \Theta - Q^{\transp}Q))\nonumber
    \\
    =& \sum_{i < n} q_i(2 Q_i \partial_{x_n} \Theta - \partial_{x_n}(Q^{\transp}Q) )
    + (\partial_{x_n}q_n) \cdot (q_n(2 Q_n \Theta - Q^{\transp}Q))
    + q_n \cdot     \partial_{x_n}(2 Q_n \Theta - Q^{\transp}Q)
    \nonumber
    \\
    =& \sum_{i < n} 2 \sqrt{\frac{2}{\pi}}q_i q_n(Q_i - Q_n)
    + x_n q_n(2 Q_n \Theta - Q^{\transp}Q)
    + q_n (2 \Theta \partial_{x_n}(Q_n) +
    2 Q_n \partial_{x_n}(\Theta)  - \partial_{x_n}(Q^{\transp}Q))
    \nonumber
    \\
    =& \sum_{i < n} 2 \sqrt{\frac{2}{\pi}}q_i q_n(Q_i - Q_n)
    + x_n q_n(2 Q_n \Theta - Q^{\transp}Q)
    + \underbrace{2 \sqrt{\frac{2}{\pi}}q_n^2 \Theta}_{\geq 0}.
    \label{eq:easy_cool_conj_1}
\end{align}
For the second term in~\eqref{eq:easy_cool_conj_1}, since \(Q_i \geq Q_n\) for any \(i \in [n]\) we have
\begin{align*}
    x_n q_n(2 Q_n \Theta - Q^{\transp}Q)
    = x_n q_n \sum_{i = 1}^n(2 Q_n Q_i - Q_i^2)
    &= x_n q_n \sum_{i = 1}^nQ_i(2 Q_n - Q_i)\\
    &\geq x_n q_n \sum_{i < n}Q_i(2 Q_n - Q_i).
\end{align*}
Thus,
\begin{align*}
    \eqref{eq:easy_cool_conj_1}
    \geq \sum_{i < n}
    \Bigg(
        2 \sqrt{\frac{2}{\pi}} q_i q_n(Q_i - Q_n)    
        + x_n q_n Q_i(2 Q_n - Q_i)
    \Bigg).
\end{align*}
Since \(Q_i - Q_n \geq 0\), if \(2 Q_n - Q_i \geq 0\) we are done. Assume otherwise. The next lemma, which relies in a classical bound on \(\erfi\) \citep[Section~7.8]{OlverLBC10a}, will be crucial for the rest of the proof.
\begin{lemma}
    \LemmaName{erfiBound}
    For any \(z > 0\), we have
    \begin{equation}
        \label{eq:erfi_ineq}
        \frac{\sqrt{\pi}}{2} \erfi(z)
        = \int_0^z e^{t^2} \diff t
        < \frac{e^{z^2} - 1}{z}.
    \end{equation}
    In particular, 
    \begin{equation*}
        Q_i < 2 \sqrt{\frac{2}{\pi}} \cdot\frac{q_i-1}{x_i}, \qquad \forall i \in [n].
    \end{equation*}
\end{lemma}
\begin{fakeproof}
    The inequality from~\eqref{eq:erfi_ineq} can be found in~\citet[Section~7.8]{OlverLBC10a}. For the second inequality, note that
    \begin{equation*}
        Q_i = \erfi\Big(\frac{x_i}{\sqrt{2}}\Big)
        \stackrel{\eqref{eq:erfi_ineq}}{<} \frac{2}{\sqrt{\pi}} \cdot \frac{e^{x_i^2/2} -1}{x_i/\sqrt{2}}
        =   2 \sqrt{\frac{2}{\pi}} \cdot \frac{q_i-1}{x_i}. \qedhere
    \end{equation*}
\end{fakeproof}
The above lemma together with \(2 Q_n - Q_i \leq 0\) implies, for each \(i \in [n]\),
\begin{equation*}
    x_n q_n Q_i (2 Q_n - Q_i)
    \geq \underbrace{\frac{x_n}{x_i}}_{\leq 1} q_n 2 \sqrt{\frac{2}{\pi}}\underbrace{(q_i - 1)}_{\leq q_i}(2 Q_n - Q_i) \geq q_n 2 \sqrt{\frac{2}{\pi}} q_i (2 Q_n - Q_i).
\end{equation*}
Thus,
\begin{align*}
    &\sum_{i < n}
    \Bigg(
        2 \sqrt{\frac{2}{\pi}} q_i q_n(Q_i - Q_n)    
        + x_n q_n Q_i(2 Q_n - Q_i)
    \Bigg)
    \\
    \geq
    &\Bigg(
        2 \sqrt{\frac{2}{\pi}} q_i q_n(Q_i - Q_n)    
        + q_n 2 \sqrt{\frac{2}{\pi}} q_i (2 Q_n - Q_i)
    \Bigg)
    = 2 \sqrt{\frac{2}{\pi}} q_n q_i Q_n \geq 0.
\end{align*}
This completes the proof of~\eqref{eq:derivative_claim} and, thus, of \Lemma{CoolConjecture}.

% \begin{proof}
% Since the functions in the stochastic differential equations in \Section{ContinuousExperts} are continuous and bounded, the stochastic integrals are well defined, vanish at time \(0\), and are martingales~\citep[Def.~2.1 and~2.3]{RevuzY99a}. In particular, \(\expect{A(t)} = 0\) for all \(t \geq 0\).  For each \(i \in [n]\) we have \(G_i(t) = B_i(t)\) since \(w^{(i)}(t) = e_i\). Thus each gain process is an independent Brownian motion. Thus,
% \begin{equation*}
%     \expect{\ContRegret(t)} = \expect{\max_{i \in [n]} B_i(t)} = \sqrt{2 T \ln n}(1 - o(1)),
% \end{equation*}
% where in the last equation we used the well-known asymptotics for the maximum of \(n\) Gaussian random variables (e.g.~\citealp[Exercise~2.11]{Wainwright19} or~\citealp[Theorem~3]{OrabonaP15a}) and the fact that \(B_i(t)\) follow a Gaussian distribution with mean zero and variance \(t\).
% \end{proof}

% Acknowledgments---Will not appear in anonymized version
\acks{We would like to thank some anonymous reviewers for their feedback and for pointing out the connection with the work of~\citet{ZhangCP22}.}

\newpage

\appendix

%%%%%%%%%%%%%%%%%%%%%%%%%%%%%%%%%%%%%%%%%%%%%%%%%%%%%%%%%%%%%%%%%%%%%%%%%%%%%%%%%%%%%%%%%%
\section{Properties of the Confluent Hypergeometric Function}
\AppendixName{M0Properties}

In this section we outline some of the main properties of the confluent hypergeometric function \(M_0\) that we use throughout the paper. Many of the properties we use in this section come from \citet[Section~2.6]{HarveyLPR20a_arxiv}.

\begin{fact}[{\citealp[Facts~2.4, 2.5, and 2.6]{HarveyLPR20a_arxiv}}]
    \label{fact:m0_properties}
    For any \(x \in \Reals\) we have
    \begin{enumerate}[(i)]
        \item \(M_0'(x) = - \frac{\sqrt{\pi}}{2 \sqrt{x}} \erfi(\sqrt{x})\)
        \item $M_0(x)$ is strictly decreasing and concave on \([0,\infty)\)
    \end{enumerate}
\end{fact}

From the above facts together with \(M_0(0) = 1\) shows us that \(M_0\) is strictly decreasing on \([0, +\infty)\) and its image over this domain is \((-\infty,1]\) (since its derivative is negative and strictly decreasing). Furthermore, the above properties of \(M_0\) allow us to derive many properties about the function \(\phi(t,x) = \sqrt{t} \cdot M_0([x]_+^2/2t)\), as we show in the next lemma.

\begin{lemma}
    \LemmaName{PropertiesOfPhi}
    Let \(t \in \Reals_{\geq 0}\) and \(x \in \Reals\). Then
    \begin{enumerate}[(i)]
        \item \(\phi(t,x)\) is concave and non-increasing in \(x\); \label{item:phi_1}
        \item $\partial_x \phi(t,x) = - \sqrt{\frac{\pi}{2}}\erfi(x/\sqrt{2t})$ for \(x > 0\); \label{item:phi_2}
        \item \(\partial_{xx} \phi(t,x) = -\frac{1}{\sqrt{t}} \exp(x^2/2t)\) for \(x > 0\); \label{item:phi_3}
        \item \(\partial_t \phi(t,x) = \frac{1}{2 \sqrt{t}} \exp([x]_+^2/2t)\);
        \label{item:phi_4}
    \end{enumerate}
\end{lemma}
\begin{fakeproof}
    Properties~\eqref{item:phi_1} and~\eqref{item:phi_2} follow directly from Fact~\ref{fact:m0_properties} and the chain rule. Property~\eqref{item:phi_3} follows from the fundamental theorem of calculus together with the chain rule since \(\sqrt{\frac{\pi}{2}}\erfi(x) = \sqrt{2}\int_0^x e^{z^2} \diff z\). For~\eqref{item:phi_4}, assume for notational simplicity only that \(x > 0\). Then,
    \begin{align*}
        \partial_t\paren*{\sqrt{t} M_0\paren*{\frac{x^2}{2t}}}
        &= \frac{1}{2\sqrt{t}} M_0\paren*{\frac{x^2}{2t}} - \frac{x^2}{2 t^{3/2}} M_0'\paren*{\frac{x^2}{2t}}
        \\
        &= \frac{1}{2\sqrt{t}} \paren*{ \exp\paren*{\frac{x^2}{2t}} - \sqrt{\pi} \frac{x}{\sqrt{2t}} \erfi\paren*{\frac{x}{\sqrt{2t}}} - \frac{x^2}{t} \frac{\sqrt{\pi 2 t}}{2 x} \erfi\paren*{\frac{x}{\sqrt{2t}}} 
        }
        \\
        &= \frac{1}{2\sqrt{t}} \exp\paren*{\frac{x^2}{2t}}.
        \qedhere
    \end{align*}
\end{fakeproof}

The next lemma gives an upper-bound to \(M_0(x^2/2)\) for \(x \geq 0\). Beyond other uses, this will be useful to upper-bound the regret bounds we derive with better-known functions.

\begin{lemma}
\LemmaName{expx2}
For every \(x \geq 0\),
\[
1-M_0(x^2/2)
    ~\geq~ \frac{\exp(x^2/2)}{x^2+1+2/x^2} \qquad \forall x \geq 0. \]
\end{lemma}
\begin{proof}
Define $f(x) = 1-M_0(x^2/2)$ and $g(x)=\exp(x^2/2)/(x^2+1+2/x^2)$.
The derivatives are
\[
f'(x) = \sqrt{\pi/2} \erfi(x/\sqrt{2})
\qquad\text{and}\qquad
g'(x) = \exp(x^2/2) \cdot \frac{x^7 - x^5 + 2 x^3 + 4x}{(x^4+x^2+2)^2}.
\]
The second derivatives are
\[
f''(x) = \exp(x^2/2)
\qquad\text{and}\qquad
g''(x) = \exp(x^2/2) \cdot
    \frac{x^{12} - x^{10} + 9x^8 + 7x^6 - 32x^4 + 8x^2 + 8}
         {(x^4+x^2+2)^3}.
\]
We will show that $f''(x) \geq g''(x)$.
By rearranging, this amounts to showing that
\begin{align*}
x^{12} - x^{10} + 9x^8 + 7x^6 - 32x^4 + 8x^2 + 8
 &~\leq~ (x^4+x^2+2)^3.
\intertext{Expanding the right-hand side, we get}
 &~=~ x^{12} + 3 x^{10} + 9 x^8 + 13 x^6 + 18x^4 + 12 x^2 + 8.
\end{align*}
The right-hand side coefficients are no smaller,
which shows that $f''(x) \geq g''(x)$ for $x \geq 0$.

By integrating and using that $f'(0)=g'(0)=0$, we obtain that $f'(x) \geq
g'(x)$ for all $x \geq 0$.
Finally, by integrating and using that $f(0)=g(0)=0$,
we obtain $f(x) \geq g(x)$ for all $x \geq 0$.
\end{proof}

The next lemma from \citet[Lemma 3.10]{HarveyLPR20a_arxiv} is used to proved that the hypergeometric potential satisfies the discrete version of the backwards heat inequality.
\begin{lemma}[\protect{\citealp[Lemma 3.10]{HarveyLPR20a_arxiv}}]
    \LemmaName{M0_ineq}
    For all $z \in [0, 1)$ and $x \in \bR$, we have
    \[
        M_0\left( \frac{(x+z)^2}{2} \right) + M_0\left( \frac{(x-z)^2}{2} \right) \geq 2 \sqrt{1-z^2} M_0\left( \frac{x^2}{2(1-z^2)} \right).
    \]
\end{lemma}

% \begin{definition}
% \DefinitionName{lambda}
% For $x \in \bR_{\geq 0}$, let $\lambda(x)>0$ be the inverse of $x \mapsto M_0(x^2/2)$,
% i.e., the unique positive solution to
% $$x = -M_0( \lambda(x)^2/2 ).$$
% \end{definition}

% \chris{Is this lemma repeated by \Lemma{lambdabound}?}
% \victor{Yes, but this version is much nicer and better placed! I've removed the other one.}

Recall from \Definition{lambda} that $\lambda$ is the non-negative inverse of $x \mapsto -M_0(x^2/2)$.
The bound on \(M_0\) from \Lemma{expx2} allows us to upper-bound \(\lambda\) as follows.

\begin{lemma}
\LemmaName{lambdabound}
Let $n \in \Reals$ be positive. Then,
\[ \lambda(n) ~\leq~ 3 + \sqrt{2 \ln(n+1)}. \]
Consequently,
$$
\lim_{n \rightarrow \infty}
\frac{ \lambda(n) }{ \sqrt{2\ln(n)} } ~\leq~ 1.
$$
\end{lemma}
\begin{proof}
Define $\ell = 3+\sqrt{2\ln(n+1)}$.
Note that $\ell^2 = 2 \ln(n+1) + 6 \ell - 9$.
So, by \Lemma{expx2},
\begin{equation}\EquationName{expx2a}
1-M_0(\ell^2/2)
    ~\geq~ \frac{\exp(\ell^2/2)}{\ell^2+1+2/\ell^2} \\
    ~=~ \frac{\exp(\ln(n+1) + 3\ell - 9/2)}{\ell^2+1+2/\ell^2}.
\end{equation}
Since $\ell \geq 3$ we have $3\ell - 9/2 \geq \ell$, and also
\begin{equation}\EquationName{expx2b}
e^{\ell} ~\geq~ \ell^2+1+2/\ell.
\end{equation}
This may be seen by a direct calculation for $\ell=3$, then observing that the second derivative of the left-hand side exceeds the second derivative of the right-hand side for $\ell \geq 3$.
Combining \eqref{eq:expx2a} and \eqref{eq:expx2b} we obtain
\[
1-M_0(\ell^2/2) ~\geq~ \exp(\ln(n+1)) ~=~ n+1.
\]
Since $-M_0$ is monotonically increasing, it follows that $\lambda(n) \leq \ell$.
\end{proof}

\section{Missing Proofs of \Section{ContMWU}}
\AppendixName{ContMWU}

\begin{lemma}
    Let \(\Phi\) be as in~\eqref{eq:logsumexp_def} and \(\eta_t \geq 0\) for all \(t \geq 0\). Then,
    \begin{equation*}
        \frac{1}{2} \sum_{i,j \in [n]} \partial_{ij} \Phi(t, x) \Sigma_{ij}(t)
        \leq \frac{\eta_t}{2},
        \qquad \forall t \geq 0, x \in \Reals^n.
    \end{equation*}
\end{lemma}
\begin{fakeproof}
    Define \(\Theta \coloneqq (\sum_{i = 1}^n \exp(\eta_t x_i))^2\). First of all, one may verify that
    \begin{equation*}
        \partial_{ii} \Phi(t,x) = \frac{1}{\Theta} \eta_t e^{\eta_t x_i}\sum_{j \neq i} e^{\eta_t x_j}
        = \frac{1}{\Theta} \eta_t \Big(\sum_{j =1}^n e^{\eta_t (x_i + x_j)} -  e^{2\eta_t x_i}\Big)
    \end{equation*}
    and that, for \(i \neq j\),
    \begin{equation*}
        \partial_{ij} \Phi(t,x) = - \frac{1}{\Theta} \eta_t e^{\eta_t (x_i + x_j)}.
    \end{equation*}
    Therefore, using that \(\Sigma_{ii}(t) =  1\) for any \(i \in [n]\) and defining \(v_i \coloneqq e^{\eta_t x_i}\) for each \(i \in [n]\)
    \begin{equation*}
        \sum_{i,j \in [n]} \partial_{ij} \Phi(t,x) \Sigma_{ij}(t)
        = \frac{\eta_t}{\Theta} \underbrace{\sum_{i,j \in [n]} e^{\eta_t (x_i + x_j)}}_{= \Theta} - \frac{\eta_t}{\Theta} \underbrace{\sum_{i,j} e^{\eta_t(x_i + x_j)} \Sigma_{ij}(t)}_{=v^{\transp} \Sigma(t) v \geq 0}
        \leq \eta_t. \qedhere
    \end{equation*}
\end{fakeproof}

\begin{lemma}
    Let \(\Phi\) be as in~\eqref{eq:logsumexp_def} and   
    \(\eta_t\) be either constant in \(t\) or of the form \(\boole{t > 0}c/\sqrt{t}\) for some \(c > 0\). Then,
    \begin{equation*}
        \partial_t \Phi(t, x) \leq \frac{\log n}{2 t \eta_t},
        \qquad \forall t \geq 0, x \in \Reals^n.
    \end{equation*}
\end{lemma}
\begin{proof}
    If \(\eta_t\) is constant as a function of \(t\), then \(\partial_t \Phi(t, x) = 0\). Otherwise, one may verify that (using the fact \(\eta_t = c/\sqrt{t}\) )
    \begin{align*}
        \partial_t \Phi(t,x)
        &= 
        \frac{1}{2 t \eta_t}\log\Big(\sum_{i = 1}^n e^{\eta_t x_i}\Big) - \frac{1}{2 \eta_t} \sum_{i = 1}^n \frac{\eta_t}{t} x_i  \frac{e^{\eta_t x_i}}{\sum_{j = 1}^n e^{\eta_t x_j}} 
        \\
        &= \frac{1}{2t} \Bigg( 
            \underbrace{\frac{1}{\eta_t} \log
            \Big( 
                \sum_{i = 1}^n e^{\eta_t x_i}
            \Big)}_{= \Phi(t, x)}
            - \underbrace{\sum_{i = 1}^n  x_i  \frac{e^{\eta_t x_i}}{\sum_{j = 1}^n e^{\eta_t C_j}}}_{=\iprodt{C}{p}}     
        \Bigg).
    \end{align*}
    Let us now show that
    \begin{equation}
        \label{eq:tight_logsumexp}
        \frac{1}{\eta_t}\log \Big(\sum_{i = 1}^n e^{\eta_t x_i} \Big)
        \leq \frac{\log n}{\eta_t} + \iprodt{x}{p},
    \end{equation}
    which completes the proof of the lemma.
    We have
    \begin{equation*}
        \frac{1}{\eta_t}\log \Big(\sum_{i = 1}^n e^{\eta_t x_i} \Big)
        = \frac{\log n}{\eta_t} + \frac{1}{\eta_t}\log \Big(\sum_{i = 1}^n \frac{1}{n}e^{\eta_t x_i} \Big).
    \end{equation*}
    Define \(z_i \coloneqq {\eta_t x_i}\) for every \(i \in [n]\). Then,
    \begin{align*}
        &\frac{1}{\eta_t}\log \Big(\sum_{i = 1}^n \frac{1}{n}e^{z_i} \Big) \leq \frac{1}{\eta_t} \sum_{i = 1}^n \frac{e^{z_i}}{\sum_{j = 1}^n e^{z_j}} z_i 
        \\
        \iff&
        \Big(\sum_{j = 1}^n e^{z_j} \Big)\log \Big(\sum_{i = 1}^n \frac{1}{n}e^{z_i} \Big) \leq  \sum_{i = 1}^n  e^{z_i} z_i
        \\
        \iff&
        \Big( \sum_{j = 1}^n \frac{1}{n}e^{z_j} \Big) \log \Big(\sum_{i = 1}^n \frac{1}{n}e^{z_i} \Big) \leq  \sum_{i = 1}^n \frac{1}{n}e^{z_i} \log(e^{z_i}),
    \end{align*}
    and this last inequality is true by the convexity of \(\alpha \in \bR_{\geq 0} \mapsto \alpha \log \alpha\). This concludes the proof of~\eqref{eq:tight_logsumexp} and, thus, of the lemma.
\end{proof}

\section{Ensuring Predictability of the Player Strategy}
\AppendixName{predictability}

In \Section{ContinuousExperts}, we required player strategies to be left-continuous in time. In fact, one could possibly loosen this assumption to only require \((p(t))_{t \geq 0}\) to be \emph{predictable} (with respect to the filtration generated by \((B(t))_{t \geq 0}\)) as defined in~\citet[Definition~IV.5.2]{RevuzY99a}. Yet, adapted left-continuous process are predictable~\citep[Lemma~7.2]{MortersP10a} and are easier to reason about directly.

One should note, however, that the player strategies defined in \eqref{eq:potential_player} may not be left-continuous. This happens due to the discontinuity when the gradient entries sum to \(0\).  For simplicity, we will discuss here how to modify player strategies generated by the potentials in~\eqref{eq:hyperg_pot} and~\eqref{eq:hyperg_pot_2} to ensure left-continuity, but similar techniques  work for players generated by other potentials. In particular, the discontinuity problems may happen only when the gradient is \(0\). Finally, we note that the exact same predictability issue arises in the continuous NormalHedge algorithm due to~\citet{Freund09a}.

Let us look at an example to see when \(p\) can be discontinuous and that it is not clear how to avoid the discontinuity in our case. Suppose \(R(s) < 0\) for \(s \in (t - \eps, t]\) for some \(t, \eps > 0\), and that \(R_1(s) > 0\) while \(R_i(s) \leq 0\) for \(s \in (t, t+ \eps)\) and \(i \in \{2, \dotsc, n\}\). Then we for all \(s \in (t - \eps, t + \eps)\) we have \(p(s) = (1/n)\ones\) if \(s \leq t\) and \(p(s) = e_1\) for \(s > t\). Ideally, we would like a smooth transition between the uniform distribution and the point mass in the first expert, but there is no clear way to enforce that. In the case of the LogSumExp potential from \Section{ContMWU}, the gradient always lives in the \emph{relative interior} of the simplex, so we never place 0 probability on any of the experts.

To ensure left-continuity of the player strategy, we can simply redefine it to be left-continuous. In our case, this will only modify the player strategy at the points of discontinuity and shall not affect our calculations. More precisely, let \((\phat(t))_{t \geq 0}\) be the player strategy as described in~\eqref{eq:potential_player} and define \((p(t))_{t \geq 0}\) by
\begin{equation*}
    p(t) \coloneqq \lim_{s \; \uparrow \; t} \phat(s),
\end{equation*}
where \((R(t))_{t \geq 0}\) is still defined in terms of \((p(t))_{t \geq 0}\). One might worry that this definition becomes circular, but note that to define \(p(t)\) we only need the values of \(R(s)\) for \(s < t\), and for \(t = 0\) we have \(R(t) = 0\). This ensures that \(p(t)\) and \(R(t)\) are well-defined. Furthermore, since \((G(t))_{t \geq 0}\) is continuous, we also have that \((A(t))_{t \geq 0}\), and thus \((R(t))_{t \geq 0}\), are continuous, even though \((p(t))_{t \geq 0}\) may not be continuous~\citep[Remark~12.1.12]{CohenE15a}.

Now \emph{by definition} we have that \(p\) is left-continuous. Moreover, the points \(t\) of discontinuity for \(p\) are the points such that \(R(t)\) enters or leaves the non-positive orthant \(\setst{x \in \Reals^n}{x \leq 0}\). Thus, we need only to ensure that any claims that explicitly use the form of \(p\) given by~\eqref{eq:potential_player} also hold in the discontinuity points. In such points it is clear we have 
\begin{equation*}
    \int_{0}^t \iprod{\nabla \Phi(s, R(s))}{\diff R(s)} = 0,
\end{equation*}
as required by~\Lemma{RegretPlayerPotFormula}. This also should not affect the calculations in the proof of~\Lemma{CoolConjecture} since we can avoid the points of non-discontinuity my small perturbations without changing the value of the sBHT significantly.

\section{Missing Proofs for \Section{QuantileRegretM0}}
\AppendixName{QuantileRegretM0}
\AppendixName{CtsPotentialNonNegative}
\begin{proofof}[\Lemma{CtsPotentialNonNegative}]
    Let \(T \geq 0\). Intuitively, we want to show that \(\Phi(T, R(T)) \geq 0\) by using \Lemma{RegretPlayerPotFormula}. However, it is not clear what should be the value of \(\Phi(0,0)\). To handle this issue, let \(\delta > 0\) and define \(\Phi_{\delta}(t,x) \coloneqq \Phi(t + \delta, x)\) for all \(t \geq 0\) and \(x \in \Reals^n\), let \(p^{(\delta)}\) be defined as in~\eqref{eq:potential_player} but replacing \(\Phi\) by \(\Phi_{\delta}\), and let \(R^{\delta}\) be the continuous regret vector of \(p^{\delta}\).  Our goal now is to show that
    \begin{equation}
        \label{eq:ctspotnonnegaitve_1}
        \Phi_{\delta}(T, R^{\delta}(T)) \geq \Phi_{\delta}(0,0) = \sqrt{\delta}~\text{almost surely}.
    \end{equation}
     Then, by taking the limit with \(\delta\) tending to \(0\) we have \(\Phi(T, R(T)) \geq 0\) almost surely. Furthermore, by a union bound we have \(\Phi(t, R(t)) \geq 0\) for all rational \(t \geq 0\), and since both \(\Phi\) and \(R\) are continuous in \(t\), this implies that \(\Phi(t, R(t)) \geq 0\) for all \(t \geq 0\) almost surely by density of the rationals. Note, however, that there is a subtlety in the step in which we take the limit with \(\delta \to 0\), since we are implicitly assuming that
     \begin{equation}
     \label{eq:limitin_argument}
     \lim_{\delta \to 0} R^{\delta}(T) = R(T)\quad \text{almost surely}.   
     \end{equation}
     Let us prove that~\eqref{eq:limitin_argument} indeed holds. Since \(R_i^{\delta}(T) = G_i(T) - \int_0^T \iprod{p^{(\delta)}(t)}{\diff G(t)}\) and \(G_i(T)\) is independent of \(\delta\) for each \(i \in [n]\), we only need to show that
     \begin{equation*}
        \lim_{\delta \to 0}\int_0^T \iprod{p^{(\delta)}(t)}{\diff G(t)} = \int_0^T \iprod{p(t)}{\diff G(t)}\quad \text{almost surely}.
     \end{equation*}
      Since \(p^{(\delta)}(t)\) is bounded and predictable (since we can assume it is left-continuous, see \Appendix{predictability}) and \(G_i\) is a continuous martingale (since it is given by a stochastic integral of continuous functions with respect to a Brownian motion) for each \(i \in [n]\), the above holds by the Dominated Convergence Theorem for stochastic integrals~\citep[Theorem~IV.2.12]{RevuzY99a}. It is also worth mentioning that we indeed have \(\lim_{\delta \downarrow 0} p^{\delta} = p\) point-wise since taking \(\delta\) to \(0\) would not make \(x^2/2(t + \delta)\) cross the negative orthant where the points of discontinuity of \(p\) may be. This completes the proof of~\eqref{eq:limitin_argument}. We now proceed with the proof of~\eqref{eq:ctspotnonnegaitve_1}.

    Since \(\Phi_{\delta}\) is separable and \(\phi(t + \delta,\cdot)\) is concave for any \(t \geq 0\), by \Lemma{RegretPlayerPotFormula} we have
    \begin{equation*}
        \Phi_{\delta}(T, R^{\delta}(T)) - \Phi_{\delta}(0,0)
        \geq \int_0^T \paren[\Big]{\partial_t\Phi_{\delta}(t, R^{\delta}(t)) + 2\sum_{i = 1}^n \partial_{ii} \Phi_{\delta}(t, R^{\delta}(t))} \diff t.
    \end{equation*}
    Note that, for any \(x \in \Reals^n\), we have \(\partial_t \Phi_{\delta}(t,x) = \sum_{i = 1}^n \partial_t \phi(t + \delta,x/2)\) and for all \(i \in [n]\) we have \(\partial_{ii} \Phi(t,x) = (1/4) \partial_{xx} \phi(t + \delta,x_i)\). Therefore,
    \begin{equation*}
        \partial_t\Phi_{\delta}(t, R^{\delta}(t)) + 2\sum_{i = 1}^n \partial_{ii} \Phi_{\delta}(t, R^{\delta}(t))
        = \sum_{i = 1}^n \paren[\Big]{ \partial_t \phi(t + \delta, R_i^{\delta}(t)) + \frac{1}{2} \partial_{xx} \phi(t + \delta, R_i^{\delta}(t))}
        \geq 0,
    \end{equation*}
    where the last equation holds since \(\partial_t \phi(t, \alpha) + (1/2) \partial_{xx}\phi(t,\alpha) \geq 0\) for any \(t > 0\) and \(\alpha \in \Reals\).
    This implies that \(\Phi_{\delta}(T, R^{\delta}(T)) \geq \Phi_{\delta}(0,0) = \sqrt{\delta}\) and \(\Phi(T, R^{}(T)) \geq 0\) by taking the limit \(\delta \to 0\).
    % The bounds on the (quantile) continuous regret now follow from \Lemma{M0ToRegret}.
\end{proofof}

\comment{
\section{Discretizing the Algorithm for Independent Experts}
\AppendixName{DiscreteAlg}

\victor{This will probably not end up in the final version of the paper}

In this section we will derive a high-probability optimal anytime regret bounds for the experts' problem when the expert costs are independent symmetric random walks. The algorithm we use is the same one as in Section~\ref{sec:disc_m0_bound} that plays according to the \(M_0\) potential with standard \(\ell_1\) normalization. That is, we have the potentials
\begin{equation*}
    \phi(t,x) \coloneqq \sqrt{t} M_0\Big( \frac{x_i^2}{2t} \Big) 
    \qquad \text{and} \qquad
    \Phi(t,x) \coloneqq \sum_{i = 1}^n \phi(t, x_i)
\end{equation*}
and the player plays according to (I am not being very careful with negative regret at this point and I'm still not sure where it could play a role)
\begin{equation*}
    p_{t+1}(i) \propto \partial_{x_i} \Phi(t + 1, R_t).
\end{equation*}
\textbf{\color{red}Right now we heavily rely on a rough deterministic bound on the regret of this algorithm.}

For the analysis, the idea is to use Taylor's series expansion, to ignore the lower-order terms, and focus on bounding the sBHT. More specifically, ignoring the high-order approximation errors (which I think should each be of order \(O(t^{-3/2})\) and using \(\approx\) every time I ignore one of these terms) and defining \(r_{t+1} \coloneqq R_{t+1} - R_t\), we have
\begin{align*}
    &\Phi(t+1, R_{t+1}) - \Phi(t, R_{t})
    \\&
    \approx \Phi(t+1, R_t) - \Phi(t, R_{t}) + \underbrace{\iprod{\nabla_x \Phi(t+1, R_t)}{ r_{t+1}}}_{= 0}
    + \frac{1}{2} \iprod{r_{t+1}}{\nabla_x^2 \Phi(t+1, R_t) r_{t+1}}
    \\&
    \approx 
    \partial_t \Phi(t, R_t) + \frac{1}{2} \iprod{r_{t+1}}{\nabla_x^2 \Phi(t+1, R_t) r_{t+1}}.
\end{align*}
The last term above is very similar to the sBHT from the continuous-time case, except by having a \(t+1\) instead of \(t\) in the second derivative of \(\Phi\). To see this similarity, first note that since \(\Phi\) is separable, we have 
\begin{align*}
    \iprod{r_{t+1}}{\nabla_x^2 \Phi(t+1, R_t) r_{t+1}}
    &= \sum_{i = 1}^n \partial_{xx} \phi(t+1, R_{t,i}) r_{t+1,i}^2
    \\
    &= \sum_{i = 1}^n \partial_{xx} \phi(t+1, R_{t,i}) (c_{t+1,i} - \iprod{p_{t+1}}{c_{t+1}})^2
    \\
    &= \sum_{i = 1}^n \partial_{xx} \phi(t+1, R_{t,i}) (e_i - p_{t+1})^{\transp} \underbrace{c_{t+1}c_{t+1}^{\transp}}_{\eqqcolon \Sigma_{t+1}}(e_i - p_{t+1})
    \\
    &= - \frac{1}{\sqrt{t+1}} \sum_{i = 1}^n \exp\Big(\frac{R_{t,i}^2}{2(t+1)}\Big)
    (e_i - p_{t+1})^{\transp} \Sigma_{t+1}(e_i - p_{t+1})
\end{align*}
and
\begin{align*}
    \partial_t \Phi(t, R_t)
    &= \frac{1}{2\sqrt{t}} \sum_{i = 1}^n \exp\Bigg(\frac{R_{t,i}^2}{2t}\Bigg)
    \geq \frac{1}{2\sqrt{t+1}} \sum_{i = 1}^n \exp\Bigg(\frac{R_{t,i}^2}{2(t+1)}\Bigg).
\end{align*}
Thus,
\begin{equation*}
    \partial_t \Phi(t, R_t) + \frac{1}{2} \iprod{r_{t+1}}{\nabla_x^2 \Phi(t+1, R_t) r_{t+1}} \geq \frac{1}{2\sqrt{t+1}} 
    \sum_{i = 1}^n \exp\Bigg(\frac{R_{t,i}^2}{2(t+1)}\Bigg)\big(1 - (e_i - p_{t+1})^{\transp} \Sigma_{t+1}(e_i - p_{t+1})\big).
\end{equation*}

\subsection{Diagonal vs Off-diagonal entries}
\AppendixName{OffDiag}

Since \(c_{t+1,i} \in \{-1,1\}\), we have that \(\Sigma_{t+1}\) has 1's in its diagonal. Moreover, if we look only at the diagonal entries (equivalently, taking \(\Sigma_{t+1} = I\)), we have a \(-O(n\sqrt{t})\) lower-bound (see Section~\ref{sec:cool_conjecture}). Thus, our main task is to lower-bound the off-diagonal terms. A bit more formally, we have
\begin{align*}
    &\sum_{i = 1}^n \exp\Bigg(\frac{R_{t,i}^2}{2(t+1)}\Bigg)\big(1 - (e_i - p_{t+1})^{\transp} \Sigma_{t+1}(e_i - p_{t+1})\big)
    \\
    =&~ 
    \underbrace{\frac{1}{2\sqrt{t+1}}\sum_{i = 1}^n \exp\Bigg(\frac{R_{t,i}^2}{2(t+1)}\Bigg)\big(1 - (e_i - p_{t+1})^T(e_i - p_{t+1})\big)}_{\geq -O(n/\sqrt{t})}
    \\&\quad
    -
    \frac{1}{2\sqrt{t+1}} \sum_{i = 1}^n \exp\Bigg(\frac{R_{t,i}^2}{2(t+1)}\Bigg)(e_i - p_{t+1})^T(\Sigma_{t+1} - I)(e_i - p_{t+1}).
\end{align*}
we only need to upper-bound---maybe with high probability when summing over \(t\)---the last term without the minus sign. Let us now make some definitions and assumptions. First, \textbf{\color{red} assume we have \(R_{t,i} \leq g(n) \sqrt{t}\)} for any \(t\) and \(i\), where \(g\) is a function that does not depend on \(t\). Therefore, we have \(\exp(\frac{R_{t,i}^2}{2(t+1)}) \leq f(n)\) for some function \(f\). Moreover, for every \(t\) define
\begin{equation*}
    M_{t} \coloneqq \Sigma_{t} - I 
    \qquad \text{and} \qquad
    v_t^{(i)} \coloneqq (e_i - p_t).
\end{equation*}
Finally, define 
\begin{equation*}
    S_T \coloneqq \sum_{t = 0}^T  \underbrace{\frac{1}{2\sqrt{t+1}} \sum_{i = 1}^n \exp\Bigg(\frac{R_{t,i}^2}{2(t+1)}\Bigg)(v_{t+1}^{(i)})^TM_{t+1}v_{t+1}^{(i)}}_{\eqqcolon Y_t} = \sum_{t = 0}^T Y_t
\end{equation*}
Let us now prove the following proposition.
\begin{proposition}
    For \(n \geq 4\) we have
    \begin{equation*}
        \prob{S_t \leq O(\ln t)~\forall t} \geq 1 - O\Big(\frac{1}{n}\Big)
    \end{equation*} 
    where the big-Oh notation in the LHS hides constants that depend on \(n\)
\end{proposition}
\begin{proof}
    First, note that \((S_t)_{t \geq 0}\) is a martingale. Indeed, if \(\{\cF_t\}_{t \geq 0}\) is the natural filtration induced by \((S_t)_{t \geq 0}\) we have \(\expect{Y_{t} | \cF_{t-1}} = 0\) since \(\expect{M_{t+1}} = 0\) (entry \(i,j\) is \(c_{t+1, i} c_{t+1,j}\), which is independent from previous rounds and has 0 expectation). Moreover, we have
    \begin{equation*}
        \abs{(v_{t+1}^{(i)})^T M_{t+1} v_{t+1}^{(i)}} \leq 4 + n 
    \end{equation*}
    The above inequality follows from Lemma~\ref{lemma:bound_quadratic_sbht} together with the fact that \(\norm{v_{t+1}^{(i)}}_2^2 \leq n\). Therefore, for $n \geq 4$
    \begin{equation*}
        \abs{Y_{t}} \leq \frac{f(n)}{2\sqrt{t+1}}
        \cdot n \cdot (4 + n) \leq \frac{f(n)n^2}{\sqrt{t+1}} \eqqcolon c_t.
    \end{equation*}
    Therefore, by Azuma's inequality we have, for any \(\eps_t > 0\), (I am ignoring \(S_0\) for simplicity, but it would add another term that depends only on \(n\) on the bound, not changing the desired result)
    \begin{equation*}
        \prob{S_t > \eps_t}
        \leq \exp\Big(\frac{- \eps_t^2}{2 \sum_{s = 0}^t c_s^2}\Big)
        \leq \exp\Big(\frac{- \eps_t^2}{2 f(n)^2 n^4 \sum_{s = 0}^t \frac{1}{t+1}}\Big)
        \leq \exp\Big(\frac{- \eps_t^2}{2 f(n)^2 n^4( 1 + \ln t)}\Big).
    \end{equation*}
    Therefore, by setting
    \begin{equation*}
        \eps_t \coloneqq  f(n) n^2 \sqrt{2(1 + \ln t)\ln(n t^2)} = O(\ln t)~\text{(hiding dep. on}~n.)
    \end{equation*}
    we get
    \begin{equation*}
        \prob{S_t > \eps_t} \leq \frac{1}{n t^2}.
    \end{equation*}
    Finally, since \(\sum_{t = 1}^\infty \frac{1}{t^2} = O(1)\), a simple union-bound gives us the desired bound.
\end{proof}

Let us now look at the regret bound we get out of this. On the event \(S_T \leq \eps_T\), we have (assume for the sake of simplicity that \(R_{T, 1} = \max_i R_{T,i}\))
\begin{align*}
    &\Phi(T, R_T) \geq - \eps_T + \sum_{t=1}^T \frac{2 - n}{\sqrt{t}}   
    \geq - \eps_T - 2 n \sqrt{T},
    \\
    \implies & \sqrt{T} \sum_{i} M_0 \Big( \frac{R_{T,i}}{2 T}\Big)
    \geq - \eps_T - 2 n \sqrt{T}
    \\
    \implies 
    & \sum_{i} M_0 \Big( \frac{R_{T,i}}{2 T}\Big)
    \geq - \frac{\eps_T}{\sqrt{T}} - 2 n
    \\ 
    \implies 
    & M_0 \Big( \frac{R_{T,1}}{2 T}\Big)
    \geq - \frac{\eps_T}{\sqrt{T}} - 3 n
    \\
    \implies 
    &\abs{R_{T,1}} \leq \lambda\Big(3n + \frac{\eps_T}{\sqrt{T}}\Big) \sqrt{T}
    \lesssim
    \sqrt{2 \ln \Big(3n + \frac{\eps_T}{\sqrt{T}}\Big) T}. 
\end{align*}
For \(\eps_T = O(\ln T)\), we have 
\begin{equation*}
    \limsup_{t} \frac{\max_{i} \abs{R_{t,i}}}{\sqrt{t}} \leq \sqrt{2 T \ln n},
\end{equation*}
but I am not sure if this is good enough. The low order terms being inside the log and multiplied by \(\sqrt{T}\) were an unwelcome surprise.

\subsection{A Deterministic Bound}

In this section we show that the algorithm of \Appendix{DiscreteAlg} achieves a regret bound
\[
Regret(T) ~\leq~ 2^{C_4 \cdot n} \sqrt{T}
\]
against adversarial (non-random) costs. This is a very weak bound, but it suffices to handle the off-diagonal entries in \Appendix{OffDiag}.

We employ the potential function approach of Cesa-Bianchi and Lugosi, modified to accommodate a time-varying potential of the form
$$
\Phi(t,u) = \psi\Big(t,\sum_{i=1}^n \phi(t,u_i)\Big).$$

\begin{theorem}[{\protect \cite[Theorem 2.1, modified]{Cesa-BianchiL06a}}]
\TheoremName{CBL}
Suppose that the algorithm has probabilities proportional to the gradient of $\Phi$. Then
\begin{align*}
\Phi(T,R_T) ~\leq~&
\sum_{t=1}^{T-1} (\partial_t \psi)\Big(t,\sum_{i=1}^n \phi(t-1,R_{t-1,i}) \Big) \\
&+ \sum_{t=1}^T (\partial_x \psi)\Big(t, \sum_{i=1}^n \phi(t-1,R_{t-1,i})\Big) \cdot \sum_{i=1}^n (\partial_t \phi)(t-1,R_{t-1,i})
+ \sum_{t=1}^T C(t,r_t)z
\end{align*}
where
$$
C(t,r_t) = \sup_{u \in \bR^n} (\partial_x \psi)\Big(\sum_{i=1}^n \phi(t,u_i)\Big) \sum_{i=1}^n \partial_{x,x} \phi(t,u_i) \cdot r_{t,i}^2.
$$
\end{theorem}

We apply this theorem with the functions
\[ \psi(t,x)=\sqrt{t} \ln \ln(C_1+x)
\qquad\text{and}\qquad
\phi(t,x) = 1-M_0(t,x^2/2t),
\]
where $C_1$ is the constant from \Lemma{M0logM0}.
The algorithm of \Appendix{DiscreteAlg} is indeed based on this potential $\Phi$ because its normalized gradient is unaffected by $\psi$ (as remarked in \cite[page 10]{Cesa-BianchiL06a}).
Some relevant derivatives are
\begin{align*}
\partial_t \psi &~=~ \frac{\ln\ln(C_1+x)}{2\sqrt{t}} \\
\partial_x \psi &~=~ \frac{\sqrt{t}}{(C_1+x)\ln(C_1+x)} \\
\partial_t \phi &~=~ - \sqrt{\pi/8} \erfi(x/\sqrt{2t})x/t^{1.5} \\
\partial_{x,x} \phi &~=~ \exp(x^2/2)/t
\end{align*}

\begin{proposition}
\label{prop:CBLC}
Let $C_2$ be the constant from \Lemma{M0logM0}. Then
$C(t,r_t) \leq C_2 n/\sqrt{t}$.
\end{proposition}

Observe that $\phi(t,x) \geq 0 ~\forall t,x$,
and $(\partial_x \psi)(t,x) \geq 0 ~\forall x \geq 0$.
On the other hand, $\partial_t\phi \leq 0 ~\forall x$.
So the second outer summation in \Theorem{CBL} is non-positive.
Applying also Proposition~\ref{prop:CBLC}, we obtain the upper bound
\begin{equation}
\EquationName{InitialUB}
\Phi(T,R_T) ~\leq~
\sum_{t=1}^{T-1} \frac{1}{2 \sqrt{t}} \ln\ln\Big(\sum_{i=1}^n \phi(t,R_{t,i}) \Big)
+ \sum_{t=1}^T \frac{C_2 n}{\sqrt{t}}.
\end{equation}

\begin{lemma}
\LemmaName{DiscPhiBound}
There is a constant $C_3>0$ such that
\begin{align}
\EquationName{DiscPhi1}
\Phi(T,R_T) &~\leq~ C_3 n \sqrt{T} \\
\EquationName{DiscPhi2}
\norm{R_T}_\infty &~\leq~ \exp(C_3 n) \sqrt{T}.
\end{align}
\end{lemma}

\begin{proof}
Suppose that \eqref{eq:DiscPhi1} holds, which is equivalent to
\[ \sum_{i=1}^n \phi(T,R_{T,i})
    ~\leq~ \exp(\exp(C_3 n)). \]
Since $\phi$ is non-negative, for each $i$ we have
\begin{equation}
\EquationName{DiscBound1}
1 - M_0(R_{T,i}^2/2T) ~\leq~ \exp(\exp(C_3 n)).
\end{equation}
Define
\[
\ell ~=~ 1 + \sqrt{2\ln(\exp(\exp(C_3 n)))}
     ~=~ 1 + \sqrt{2} \exp(C_3 n/2)
     ~\leq~ \exp(C_3 n),
\]
if $C_3$ is sufficiently large.
Following the proof of \Lemma{lambdabound},
we have 
\begin{equation}
\EquationName{DiscBound2}
\exp(\exp(C_3 n)) ~\leq~ 1 - M_0(\ell^2/2).
\end{equation}
Combining inequalities \eqref{eq:DiscBound1} and \eqref{eq:DiscBound2}, we obtain
$R_{T,i}/\sqrt{T} \leq \ell$, which then implies \eqref{eq:DiscPhi2}.

It remains to prove \eqref{eq:DiscPhi1}, which we will do by induction.
From \eqref{eq:InitialUB} we have 
\begin{align*}
\Phi(T,R_T)
&~\leq~
\sum_{t=1}^{T-1} \frac{1}{2 \sqrt{t}} \ln\ln\Big(\sum_{i=1}^n \phi(t,R_{t,i}) \Big)
+ \sum_{t=1}^T \frac{C_2 n}{\sqrt{t}}.
\intertext{By induction and \eqref{eq:DiscPhi2} (for $t<T$) we have}
&~\leq~
\sum_{t=1}^{T-1} \frac{1}{2 \sqrt{t}} \ln\ln\big(n \cdot \phi(t,\exp(C_3 n) \sqrt{t}) \big)
+ 2 \cdot C_2 n \sqrt{T} \\
\intertext{Using $\phi(t,x) \leq \exp(x^2/2t)$,}
&~\leq~
\sum_{t=1}^{T-1} \frac{1}{2 \sqrt{t}} \ln\ln\big(n \exp(\exp(2 C_3 n)/2) \big)
+ 2 \cdot C_2 n \sqrt{T}.
\end{align*}
\end{proof}

\newcommand{\Maximizer}{I}
\begin{proofof}{Proposition \ref{prop:CBLC}}
\begin{align*}
C(t,r_t)
&~=~ \sup_{u \in \bR^n} \frac{\sqrt{t}}{\big(C_1+\sum_{i=1}^n \phi(t,u_i)\big) \cdot\ln\big(C_1+\sum_{i=1}^n \phi(t,u_i)\big)} \sum_{i=1}^n \frac{e^{u_i^2/2}}{t} \cdot r_{t,i}^2
\intertext{For each $u$, we can pick $\Maximizer(u) \in \argmax_{i} e^{u_i^2/2} r_{t,i}^2$, to obtain}
&~\leq~ \frac{1}{\sqrt{t}} \sup_{u \in \bR^n} \frac{1}{\big(C_1+\phi(t,u_{\Maximizer(u)})\big) \cdot\ln\big(C_1+\phi(t,u_{\Maximizer(u)})\big)} \cdot n \cdot e^{u_{\Maximizer(u)}^2/2} \cdot r_{t,\Maximizer(u)}^2
\intertext{Now applying \Lemma{M0logM0},}
&~\leq~ \frac{n}{\sqrt{t}} \sup_{u \in \bR^n} \frac{1}{\exp(u_{\Maximizer(u)}^2/2t)/C_2} \cdot e^{u_{\Maximizer(u)}^2/2} \cdot r_{t,\Maximizer(u)}^2 \\
&~=~ \frac{n}{\sqrt{t}} \sup_{u \in \bR^n} C_2 \cdot r_{t,\Maximizer(u)}^2 \\
&~\leq~ \frac{C_2 n}{\sqrt{t}},
\end{align*}
since $r_{t,i}^2 \leq 1$ for all $i$.
\end{proofof}
}

\bibliography{ref.bib}
\end{document}